\documentclass[fleqn,5p]{elsarticle} 

\usepackage{lineno,hyperref}
\hypersetup{pdfauthor=author} 
\modulolinenumbers[5]

\journal{Robotics and Autonomous Systems}

\usepackage[T1]{fontenc}
\usepackage[utf8]{inputenc}
\usepackage{bm}
\usepackage{graphics} 
\DeclareMathAlphabet{\mathcal}{OMS}{cmsy}{m}{n}
\usepackage{makecell}
\usepackage{tabularx}
\usepackage{adjustbox}
\usepackage{multirow}
\usepackage{array}
\newcolumntype{P}[1]{>{\centering\arraybackslash}p{#1}}
\usepackage{comment}
\usepackage{times} 
\usepackage{amsmath} 
\usepackage{stmaryrd}
\usepackage{amsfonts}
\usepackage{amssymb}  
\usepackage[version-1-compatibility]{siunitx}
\usepackage{enumitem}
\usepackage{algorithm}
\usepackage{algorithmic}

\newtheorem{lemma}{Lemma}
\newtheorem{theorem}{Theorem}
\newtheorem{assumption}{Assumption}
\newdefinition{remark}{Remark}
\newdefinition{definition}{Definition}
\newproof{proof}{Proof}

\usepackage[english]{babel}
\usepackage[dvipsnames]{xcolor}

\usepackage{graphicx}
\usepackage{xstring}
\usepackage{forloop}

\makeatletter
\newcommand*\bigcdot{\mathpalette\bigcdot@{1}}
\newcommand*\bigcdot@[2]{\mathbin{\vcenter{\hbox{\scalebox{#2}{$\m@th#1\bullet$}}}}}
\makeatother

\begin{document}
\newcommand{\cmmnt}[1]{}
\newcommand\blfootnote[1]{%
  \begingroup
  \renewcommand\thefootnote{}\footnote{#1}%
  \addtocounter{footnote}{-1}%
  \endgroup
}

\begin{frontmatter}

\title{Marine Locomotion: A Tethered UAV-Buoy System with Surge Velocity Control}



\author[1]{Ahmad {Kourani}}
\ead{ahk42@mail.aub.edu}

\author[2]{Naseem {Daher}\corref{cor1}}
\cortext[cor1]{Corresponding author.}
\ead{nd38@aub.edu.lb}

\address[1]{Vision and Robotics Lab, Department of Mechanical Engineering, American University of Beirut, Riad El-Solh, Beirut, 1107 2020, Lebanon}
\address[2]{Vision and Robotics Lab, Department of Electrical and Computer Engineering, American University of Beirut, Riad El-Solh, Beirut, 1107 2020, Lebanon}

\begin{abstract}
Unmanned aerial vehicles (UAVs) are reaching offshore. In this work, we formulate the novel problem of a marine locomotive quadrotor UAV, which manipulates the surge velocity of a floating buoy by means of a cable. 
The proposed robotic system can have a variety of novel applications for UAVs where their high speed and maneuverability, as well as their ease of deployment and wide field of vision, give them a superior advantage. In addition, the major limitation of limited flight time of quadrotor UAVs is typically addressed through an umbilical power cable, which naturally integrates with the proposed system. 
A detailed high-fidelity dynamic model is presented for the buoy, UAV, and water environment. In addition, a stable control system design is proposed to manipulate the surge velocity of the buoy within certain constraints that keep the buoy in contact with the water surface. 
Polar coordinates are used in the controller design process since they outperform traditional Cartesian-based velocity controllers when it comes to ensuring correlated effects on the tracking performance, where each control channel independently affects one control parameter.
The system model and controller design are validated in numerical simulation under different wave scenarios.
\end{abstract}


\end{frontmatter}




\section{Introduction} \label{introduction}

%
Aerial drones are finding their way into different sectors of the industry, including construction \cite{Yang2018a,Krizmancic2020}, agriculture \cite{Tripicchio2015}, package delivery \cite{amazon}, inspection and maintenance \cite{Kovac2019}, to name a few, in which drones not only independently fly in the air, but also physically interact with the environment.
In terms of activities, unmanned aerial vehicles (UAVs) can move slung payloads in solo \cite{Sreenath2013} or cooperatively \cite{Srikanth2011} for transportation tasks,
they can interact and collaborate with unmanned ground vehicles (UGVs), and they can be equipped with robot manipulators to achieve different geometric configurations \cite{Naldi2012} or to cooperatively manipulate other objects \cite{Yang2018a, Staub2018, Nguyen2019}.

A common medium for UAVs to interact with their environment is through a tether \cite{Lupashin2013StabilizationTether,Nicotra2014}, as it can have a variety of interesting applications including the transmission of power, forces, and data. Tethered UAVs were studied for stability and control while maintaining positive cable tension in \cite{Nicotra2017}. 
%
%
In fact, the tether can enhance the stability and performance of an aerial vehicle \cite{Sandino2013TetheredHoveringStability,Sandino2014TetheredHoveringPerformance}, and can enable estimating the UAV's relative pose to the tether anchor using on-board inertial sensors only \cite{Lupashin2013StabilizationTether}, which can be upgraded with additional sensors, such as force sensors and encoders, for better performance \cite{Tognon2020AerialRobotsTethers}.
A major contribution to this field was presented in \cite{Tognon2017}, which proved that tethered aerial robotic systems are deferentially flat with respect to two outputs: the link's elevation angle and either the vehicle's attitude or the longitudinal link force.
\par
%
Aerial drones are well-suited for applications that meet the 4D criteria: dull, dirty, distant, and dangerous \cite{Weichenhain2019}. As such, the offshore oil industry and offshore wind-farms are excellent candidates for their adoption, given the potential of drones to become the go-to technology in assets inspection and infrastructure maintenance \cite{Weichenhain2019}.
For example, traditional offshore solutions such as inspecting offshore wind-farms entail moving a vessel, which is expensive and requires a human crew, unlike the deployment of drones that can significantly save cost and time \cite{Kovac2019}.
UAVs can also perform sensing jobs, place sensors, and perform on-sight repairs and maintenance \cite{Kovac2019}.
Furthermore, offshore applications of UAVs make it more likely for aviation authorities to permit their utilization, since they are deployed faraway from human populations \cite{Murison2018}. 
\par
%
Due to their limited power capacity and flight time, the interaction of UAVs with the marine environment is still in its early stages. Current uses are mainly limited to information gathering such as transmitting visual feedback to human operators, targeting the locations of floating objects for their retrieval \cite{Miskovic2014}, and generating and transmitting full area maps and path-planning for other agents to perform rescue missions \cite{Ozkan2019}. Physical interaction is limited to low-power applications such as landing assistance on a rocking ship \cite{SoRyeokOh2006}, power-feeding the UAV through a cable \cite{Talke2018,choi2014tethered}, and sensing jobs \cite{Tognon2019}.
\par
Although unmanned surface vehicles (USVs) are naturally-suited robots in marine environments, UAVs can outperform them in certain aspects that make it more practical to adopt UAV-based marine solutions and applications.
\sloppy First, UAVs are advantageous over USVs in terms of their field of vision (bird's-eye view), ease of deployment, and maneuverability, all of which give UAVs the advantage while performing tasks in unstructured and hard-to-reach areas, and tasks that require precision and quick deployment. In addition, UAVs are especially advantageous in rivers since they can follow shorter paths above land and avoid in-water obstacles and waterfalls.
Furthermore, it is more challenging to deploy and retrieve USVs since they require direct access to the water surface, whereas UAV's benefit from their vertical takeoff and landing (VTOL) capabilities to be independently deployed from anywhere.
This fact highlights the advantage that UAVs have in addressing the issue of the limited and expensive free-space on offshore structures, vehicles, and coastal strips that stand to benefit from deploying aerial robotic system solutions.
\par
From the above literature survey and discussion, it is evident that having an integrated system that incorporates an umbilical power cable can open the door in front of a whole new level of UAV marine applications.
An analogous system was investigated in \cite{Papachristos2014}, in which an unmanned ground vehicle carries the power source while following a tethered UAV.
Furthermore, an optimal length and tension design of a cable that links a UAV and USV was provided in \cite{Talke2018}. The optimization problem minimizes both fouling (cable entanglement or jamming) and excessive downforce on the UAV during dynamic heaves, which boosts the power capacity of the UAV and simultaneously optimizes the dynamic performance of the coupled system.
In addition, the employment of USVs as a landing platform has been studied in the literature; for instance, a coupled UAV$-$USV system was proposed in \cite{Shao2019UAV-USV}, where the USV is equipped with an expendable landing deck for additional safety, and the system serves as a foundation for further collaborative tasks.
\par
Leveraging the technological advances in UAVs technology in terms of robustness, accuracy, operational cost, and lately, power efficiency,
and motivated by applications requiring fast action with minimal water surface disruption \cite{Saleh2019},
we are proposing the employment of a quadrotor UAV to manipulate a passive floating object via a cable, whereby the quadrotor performs the function of a locomotive.
The umbilical power line solution naturally integrates into this system, where the cable can be used for both force and power transmission. 
The hereby proposed problem generalizes the fixed-point tether described in \cite{Nicotra2014} to a moving-frame tether, namely planar motion in the horizontal and vertical directions, and is subject to additional constraints such as maintaining contact with the water surface. The formulated problem and proposed solution pave the way in front of numerous UAV$-$USV interaction applications, some of which are described next.
\par
%
%
The proposed marine locomotive UAV system can be used in coordination with nearby ships and marine structures to increase their maneuverability and decrease their response time, as well as nearshore and other water surfaces such as rivers and across waterfalls.
The proposed system can help in performing a variety of tasks including rescue operations, floating objects recovery, building and inspecting marine structures, water samples collection, delicately placing and relocating marine sensors and buoys with minimal water surface disruption, fishing activities, and water surface clean-up efforts, to name a few. In this context, we are motivated by the marine application in \cite{Saleh2019}, which proposed a sensor for measuring oil slick thickness during marine oil spill events. The proposed sensor is fixed to a floating buoy that is pulled by another vessel to skim the water surface. One main challenge in the proposed solution lies in the vessel's motion ahead of the sensor, which tends to disturb the oil layer and thus reduces the measurement's accuracy.
\par
%

%
\par
%
%
This paper offers several technical contributions. First, the novel problem of the marine locomotive UAV is formulated, which paves the way for further research into the interaction between UAVs and the marine environment.
Second, the system is defined in a sea/ocean environment that accounts for the presence of gravity waves and surface current, which naturally extends to wave-free water mediums. 
Third, the buoy and quadrotor UAV coupled dynamics are modelled with high fidelity using the Lagrangian formulation with appropriate constraints for the tethered UAV$-$buoy system. 
Forth, the attainable equilibrium states are derived with a proper definition of the system's operational limits and constraints in terms of cable tension, water surface contact, and buoy velocity.
Fifth, we design and validate a buoy surge velocity control system, supervised by a state machine that switches between operational modes, which results in accurate tracking performance even in the presence of disturbing waves, water currents, and feedback noise, while reducing the system's energy consumption by maintaining a constant UAV altitude. The controller relies on polar coordinates with respect to the buoy's reference frame to realize correlated tracking, which outperforms traditional Cartesian-based and unsupervised UAV-only velocity controllers that do not lend themselves well to this application.
Lastly, we make available a physics engine that can be used for simulating tethered UAV$-$buoy locomotives via a custom-built simulator\footnote{\texttt{github.com/KouraniMEKA/Marine-Locomotive-UAV}}.
\par
%
%
The rest of this paper is structured as follows. A detailed description of the tethered UAV$-$buoy system model is presented in Section \ref{sec_modeling}. The designed control system is detailed in Section \ref{sec_controller_design}. Section \ref{sec_simulation} presents numerical simulation results that demonstrate the validity of the derived system model and the effectiveness of the designed controller. Section \ref{sec_practical_considerations} discusses some practical considerations for the implementation of the proposed system, and finally Section \ref{sec_conclusion} concludes the paper and provides an outlook into future work. \par
%
%
%
\section{Tethered UAV-Buoy System Model} \label{sec_modeling} 
The dynamic model of the tethered UAV$-$buoy system requires the integration of multiple domains including the fluid medium; the dynamics of the floating buoy, the UAV, and the cable; and the combined system of rigid bodies. In this section, we introduce the required subsystems to formulate the problem on hand. \par
\subsection{Preliminaries} \label{subsec_preliminaries}
This section introduces some of the critical notations that are used throughout the paper. We let the set of positive-real numbers $\{x \in \mathbb{R}\,|\,x>0 \}$ be denoted as $\mathbb{R}_{>0}$, and the set of non-negative real numbers $\{x \in \mathbb{R}\,|\,x\geq 0 \}$ be denoted as $\mathbb{R}_{\geq0}$.
Also, let $s_{\bigcdot}$, $c_{\bigcdot}$, and $t_{\bigcdot}$ respectively be the sine, cosine, and tangent functions for some angle ($\bigcdot$).
In addition, let $\|\cdot\|$ denote the $L_2$ norm.
%
%
\subsection{Problem Definition} \label{subsec_problem_definition}
Consider the two-dimensional (2D) space in the water vertical plane where the problem is set up as shown in Fig.~\ref{fig_Buoy_UAV_Annotations}, and let $\mathcal{W}=\{x,y\}$ represent the inertial frame of reference whose origin, $\mathcal{O}_{\mathrm{I}}$, is at the local mean sea level horizontal line.
%
\begin{figure}
\centerline{\includegraphics[width=3.45in]{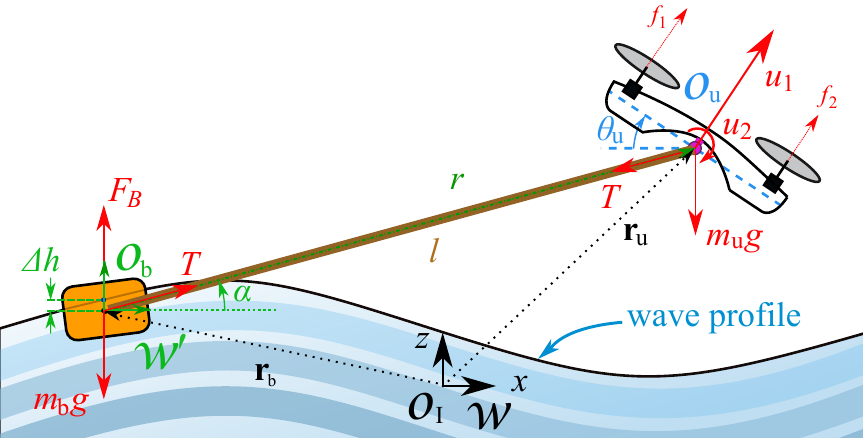}}
\caption{Planar model of a quadrotor UAV pulling a floating buoy through a tether.}
\label{fig_Buoy_UAV_Annotations}
\end{figure}
%
Considering the tethered UAV$-$buoy system depicted in Fig.~\ref{fig_Buoy_UAV_Annotations},
the buoy is physically connected to the UAV by means of a cable of length $l \in \mathbb{R}_{>0}$, forming an angle $\alpha \in (0,\pi)$ with the positive $x$-axis, which is defined as the elevation angle. 
%
Let $\mathbf{r}_{\mathrm{b}}=\{x_{\mathrm{b}},z_{\mathrm{b}}\} \in \mathbb{R}^2$ and $\mathbf{r}_{\mathrm{u}}=\{x_{\mathrm{u}},z_{\mathrm{u}}\} \in \mathbb{R}^2$ respectively be the coordinates of the buoy's center of mass, ($\mathcal{O}_{\mathrm{b}}$), and that of the UAV, ($\mathcal{O}_{\mathrm{u}}$), in $\mathcal{W}$; for ease of use, we set $V:=\dot{x}_{\mathrm{b}}$ to depict the buoy's horizontal velocity. Let $\mathcal{B}_{\mathrm{b}}$ and $\mathcal{B}_{\mathrm{u}}$ be the body-fixed reference frames of the buoy at $\mathcal{O}_{\mathrm{b}}$, and of the quadrotor at $\mathcal{O}_{\mathrm{u}}$, respectively.
The floating buoy has a volume $ \curlyvee_{\mathrm{b}}\in \mathbb{R}_{>0}$, a bounded mass $m_{\mathrm{b}}\in (0,\rho_{\mathrm{w}} \curlyvee_{\mathrm{b}})$ moment of inertia $J_{\mathrm{b}}\in \mathbb{R}_{>0}$ in $\mathcal{B}_{\mathrm{b}}$; the quadrotor UAV has a mass $m_{\mathrm{u}}$ and a moment of inertia $J_{\mathrm{u}}\in \mathbb{R}_{>0}$ in $\mathcal{B}_{\mathrm{u}}$.
Also let the orientation, measured clockwise, of $\mathcal{B}_{\mathrm{b}}$ and $\mathcal{B}_{\mathrm{u}}$ with respect to $\mathcal{W}$ be described by the angles $\theta_{\mathrm{b}}$ and $\theta_{\mathrm{u}} \in (-\pi,\pi]$, respectively. 
Let $\bm{V}_{\mathrm{b}}=\{u_{\mathrm{b}},w_{\mathrm{b}}\} \in \mathbb{R}^2$ and $\Omega_{\mathrm{b}} \in \mathbb{R}$ be the linear and angular velocities of the buoy in $\mathcal{B}_{\mathrm{b}}$, respectively.
Furthermore, let the translational rotation matrix from any body frame to $\mathcal{W}$ be described as:
\begin{equation} \label{eq_Rotation_Matrix_e2b_generalized} 
    \mathbf{R}_{\bigcdot}=
    \begin{bmatrix}
        c_{\bigcdot} &        -s_{\bigcdot} \\
        s_{\bigcdot} &  \;\;\; c_{\bigcdot} \\
    \end{bmatrix}.
\end{equation} 
\par
Both the buoy and the UAV are subject to gravitational acceleration, $\textsl{g}$, and cable tension, $T, \in \mathbb{R}_{\geq0}$. 
Moreover, the buoy is subjected to hydrostatic and hydrodynamic forces that are described later,
and the UAV propulsion can be simplified to only include the total thrust $u_1 \in \mathbb{R}_{\geq0}$, and a single torque that induces a pitch motion $u_2 \in \mathbb{R}$ since the motion of the system is restricted under the scope of this work to the vertical plane. Considering the relatively faster response of the UAV actuators as compared to the UAV itself, their dynamics are neglected in modeling. \par
\begin{remark} \label{rem_buoy_density}
    The scope of this work covers the manipulation of a floating buoy, thus the buoy's average density should not exceed the density of water, which is achieved with the constraint $m_{\mathrm{b}} \in (0,\rho_{\mathrm{w}} \curlyvee_\mathrm{b})$.
\end{remark}
\begin{assumption} \label{assump_cable}
    The cable is inextensible; it is attached to the buoy's center of mass at one end and to the UAV's center of mass at the other via revolute joints to prevent moment transmission; and for relatively small systems considered in this work, it can be of negligible mass. Considerations for heavy cables (slung payload) are provided in Section~\ref{sec_practical_considerations}. 
\end{assumption}
%
\subsection{Water Medium Model} \label{subsec_water_medium}
The water medium under consideration here is the sea/ocean, where the main aspects of interest are gravity waves and water surface current. \par
\subsubsection{Gravity Wave Model}
\begin{assumption} \label{assump_linear_waves_theory}
    In the considered problem environment, the water depth is assumed to be much larger than the wavelength of gravity waves, which are assumed to be of moderate height. This permits adopting linear wave theory in this work \cite{faltinsen1990}. In addition, the wave direction is limited to be in the vertical ($x-z$) plane. 
\end{assumption} \par
Based on Assumption \ref{assump_linear_waves_theory}, the water elevation variation, $\zeta$, at time $t$ and horizontal position $x$ due to gravity waves is statistically described as:
\begin{equation} \label{eq_water_surface_elevation}
    \zeta(x,t) = \sum_{n}^{N} A_n \sin(d_n \omega_n t - k_n x + \sigma_n),
\end{equation}
\noindent where $A_n$, $\omega_n$, $k_n \in \mathbb{R}_{\geq 0}$, $d_n \in \{-1,1\}$, and $\sigma_n \in (-\pi,\pi]$ are respectively the wave amplitude, circular frequency, wave number, wave direction coefficient, and random phase angle of wave component number $n \in S_n$ with $S_n = \{1 \leq n \leq N \, | \, N  \in \mathbb{N}\}$. Furthermore, based on Assumption \ref{assump_linear_waves_theory}, the wave number in deep water is given by the dispersion relation as $k_n = \omega_n^2/\textsl{g}$.
The horizontal and vertical fluid particles' wave-induced velocities can be prescribed as \cite{faltinsen1990}:
\begin{equation} \label{eq_wave_vel}
    \begin{split}
        v^w_{x}(x,z,t) & = \sum_{n}^{N} d_n \omega_n A_n e^{k_n z}   \sin(d_n \omega_n t - k_n x + \sigma_n),\\
        v^w_{z}(x,z,t) & = \sum_{n}^{N} d_n \omega_n A_n e^{k_n z}   \cos(d_n \omega_n t - k_n x + \sigma_n).\\
    \end{split}
\end{equation}
\noindent where $\omega_n$ relates to the wave period, $\mathrm{T}_n$, via $\omega_n = 2\pi / \mathrm{T}_n$. 
\par
\subsubsection{ Water Current}
For brevity purposes, a simple yet comprehensive model of the water current is adopted. The water surface current, acting in the horizontal $x$-direction, is given as:
\begin{equation} \label{eq_current}
    U_{\mathrm{c}} = U_{\mathrm{l}} + U_{\mathrm{s}}, 
\end{equation}
\noindent where $U_{\mathrm{l}} \in \mathbb{R}$ is the lumped sum of different water current components, and $U_{\mathrm{s}}\in \mathbb{R}$ is the component generated from Stokes drift \cite{Fossen1995}.
The Stokes drift velocity is one component that emerges from nonlinear wave theory, and is defined as the average transport velocity of a wave over one period:
\begin{equation} \label{eq_stokes_drift_vel}
    U_{\mathrm{s}}(z) = \sum_{n}^{N} d_n A^2_n \omega_n k_n e^{2 k_n z}.
\end{equation}
\par
%
\subsection{Buoy's Dynamic Model} \label{subsec_buoy_model}
A floating buoy is subjected to different types of forces, with the main ones being radiation, damping, and restoration forces.
The radiation forces consist of the added mass and added damping. In this section, a detailed description of the buoy model is presented after introducing the following assumptions. \par
\begin{assumption} \label{assump_buoy_body_axes}
    The axes of the buoy's body frame coincide with its principle axes of inertia, which is a common practice to simplify the modeling of marine vehicles \cite{Fossen1995}.
\end{assumption} \par
\begin{assumption}\label{assump_neglected_air_friction}
    The water-buoy friction dominates the energy dissipation in the system, and the system is assumed to operate under moderate weather conditions, thus energy losses due to air drag are neglected.
\end{assumption}
Considering the buoy dynamics in $\mathcal{B}_{\mathrm{b}}$ with the state vector, $\bm{\nu}_{\mathrm{b}} = [u_{\mathrm{b}},w_{\mathrm{b}},\Omega_{\mathrm{b}}]^{\intercal}$, and applying Newton's second law of motion yields: 
\begin{equation} \label{eq_buoy_body_Newton}
    \mathbf{M}'_{\mathrm{b}} \dot{\bm{\nu}}_{\mathrm{b}} + \mathbf{C}'_{\mathrm{b}} \bm{\nu}_{\mathrm{b}} + \mathbf{D}'_{\mathrm{b}} \tilde{\bm{\nu}}_{\mathrm{b}} + \mathbf{G}'_{\mathrm{b}} = \bm{\tau}'_{\mathrm{b}},
\end{equation}
\noindent where $\mathbf{M}'_{\mathrm{b}}$, $\mathbf{C}'_{\mathrm{b}}$, and $\mathbf{D}'_{\mathrm{b}} \in \mathbb{R}^{3\times3}$ are respectively the buoy's inertia, Coriolis, and damping matrices expressed in $\mathcal{B}_{\mathrm{b}}$; the relative velocity vector is defined as $\tilde{\bm{\nu}}_{\mathrm{b}} = \bm{\nu}_{\mathrm{b}} - [U_{\mathrm{c}}+v_x^w c_{\theta_{\mathrm{b}}}-v_z^w s_{\theta_{\mathrm{b}}},v_z^w c_{\theta_{\mathrm{b}}} +v_x^w s_{\theta_{\mathrm{b}}},0]^{\intercal}$; $\mathbf{G}'_{\mathrm{b}} \in \mathbb{R}^3$ is the gravitational forces and moments vector; and $\bm{\tau}'_{\mathrm{b}} \in \mathbb{R}^3$ includes external forces and moments. \par
%
%
The inertia matrix is $\mathbf{M}'_{\mathrm{b}}=diag(m_{\mathrm{b}}+a_{11},m_{\mathrm{b}}+a_{33},J_{\mathrm{b}}+a_{55})$, where $a_{11}$, $a_{33}$, and $a_{55} \in \mathbb{R}_{\geq0}$ are the surge, heave, and pitch rate components of the generalized added mass matrix. 
The added mass can be described as the amount of fluid that is accelerated with the body, and can be written as a function of the buoy's mass and moment of inertia. 
Furthermore, for low frequency motion, $a_{11} \approx 0.05 m_{\mathrm{b}}$, and $a_{33} \approx m_{\mathrm{b}}$ \cite{Fossen1995}.
Given $\mathbf{M}'_{\mathrm{b}}$, the Coriolis matrix is calculated as follows:
\begin{equation} \label{eq_buoy_body_Coriolis}
    \mathbf{C}'_{\mathrm{b}}=
    \begin{bmatrix}
            0        &          0           & (m_{\mathrm{b}} + a_{33}) w_{\mathrm{b}}  \\
            0        &  -(m_{\mathrm{b}} + a_{11}) u_{\mathrm{b}} & 0\\
        -a_{33} w_{\mathrm{b}}  &   a_{11} u_{\mathrm{b}}         & 0 \\
    \end{bmatrix}. 
\end{equation}
\noindent The total damping term of the buoy in $\mathcal{B}_{\mathrm{b}}$ is expressed as:
\begin{equation} \label{eq_buoy_body_damping}
    \mathbf{D}'_{\mathrm{b}} = \mathbf{D}_{\mathbf{P}}  + \mathbf{D}_{\mathbf{S}}  + \mathbf{D}_{\mathbf{W}} ,
\end{equation}
\noindent where $\mathbf{D}_{\mathbf{P}}=diag(b_{11},b_{33},b_{55}) \in \mathbb{R}^{3\times3}$ is the radiation induced potential damping matrix with surge, heave, and pitch components, and $\mathbf{D}_{\mathbf{S}}=diag(D_{\mathrm{S},1},D_{\mathrm{S},2},D_{\mathrm{S},3}) \in \mathbb{R}^{3\times3}$ is the skin friction matrix, calculated as:
\begin{equation} \label{eq_buoy_body_friction_coeff}
    D_{\mathrm{S},i} = C_{\mathrm{S},i} \mathrm{A}_{\mathrm{wt}} \frac{1}{2} \rho_{\mathrm{w}} |\tilde{\nu}|, \quad i=1,2,
\end{equation}
\noindent where $C_{\mathrm{S},i} \in \mathbb{R}_{>0}$ is a constant, $\mathrm{A}_{\mathrm{wt}} \in \mathbb{R}_{\geq 0}$ is the buoy's wetted area, and $D_{\mathrm{S},3} \in \mathbb{R}_{\geq0}$.
$\mathbf{D}_{\mathbf{W}} \in \mathbb{R}^{3\times3}$ is the wave drift damping matrix, which will be dropped from (\ref{eq_buoy_body_damping}) since its effect is already included in the Stokes drift velocity in (\ref{eq_stokes_drift_vel}). 
Assuming that the buoy, with a mean immersed height $\bar{h}_{\mathrm{im}}$, vertically oscillates in the $x-z$ plane at $\omega_{\mathrm{o},3}\in \mathbb{R}_{\geq0}$, such that $\omega_{\mathrm{o},3} < 0.2 \sqrt{\textsl{g} /\bar{h}_{\mathrm{im}}}$, which is practical for the problem on hand,  we have $b_{33} \approx 2 m_{\mathrm{b}} \omega_{\mathrm{o},3}$. Moreover, the potential damping coefficients in the horizontal plane vanish at both limits $0$ and $\infty$ of the oscillation frequency, thus the potential damping in the $x$-direction is $b_{11} \approx 0$ \cite{Fossen1995}. 
\par
Referring to Assumption~\ref{assump_cable}, the buoy dynamics in (\ref{eq_buoy_body_Newton}) can be expressed in $\mathcal{W}$ with the state vector, $\bm{\eta}_{\mathrm{b}} = [x_{\mathrm{b}},z_{\mathrm{b}},\theta_{\mathrm{b}}]^{\intercal}$, as:
\begin{equation} \label{eq_buoy_inertial_Newton}
    \mathbf{M}_{\mathrm{b}} \ddot{\bm{\eta}}_{\mathrm{b}} + \mathbf{C}_{\mathrm{b}} \dot{\bm{\eta}}_{\mathrm{b}} + \mathbf{D}_{\mathrm{b}} \tilde{\dot{\bm{\eta}}}_{\mathrm{b}} + \mathbf{G}_{\mathrm{b}} = \bm{\tau}_{\mathrm{b}},
\end{equation}
\noindent where $\mathbf{D}_{\mathrm{b}}$, $\mathbf{C}_{\mathrm{b}}$, and $\mathbf{D}_{\mathrm{b}} \in \mathbb{R}^{3\times3}$ are respectively the buoy's inertia, Coriolis, and damping matrices expressed in $\mathcal{W}$;
the relative velocity vector is defined as $\tilde{\bm{\dot{\eta}}}_{\mathrm{b}} = [\tilde{\dot{\eta}}_{\mathrm{b},1},\tilde{\dot{\eta}}_{\mathrm{b},2},\tilde{\dot{\eta}}_{\mathrm{b},3}]^{\intercal} = \bm{\dot{\eta}}_{\mathrm{b}} - [U_{\mathrm{c}}+v_x^w,v_z^w,0]^{\intercal}$; $\mathbf{G}_{\mathrm{b}}$ and $\bm{\tau}_{\mathrm{b}}$ are respectively the vectors of the buoy's gravitational and other external forces and moments in $\mathcal{W}$ expressed as:
\begin{equation} \label{eq_buoy_inertial_forces}
        \mathbf{G}_{\mathrm{b}} = [0,m_{\mathrm{b}} \textsl{g},0]^{\intercal},\qquad
        \bm{\tau}_{\mathrm{b}} = [T c_{\alpha},F_{\mathrm{B}} + T s_{\alpha},F_{\mathrm{rs}}]^{\intercal},
\end{equation}
\noindent where $F_{\mathrm{B}} = \rho_{\mathrm{w}} \textsl{g} \curlyvee_{\mathrm{im}}$ is the buoyancy force, $\curlyvee_{\mathrm{im}} \in [0,\curlyvee_\mathrm{b}]$ is the immersed volume of the buoy, and $F_{\mathrm{rs}} = f_{\mathrm{p}}s_{\theta_{\mathrm{u}}}$ is the pitch restoring moment with $f_{\mathrm{p}} \in \mathbb{R}$ being the buoy's pitch restoring moment coefficient.
We also define:
\begin{equation} \label{eq_buoy_body_to_inertial_trans}
    \begin{split}
        & \mathbf{M}_{\mathrm{b}} = (\mathbf{R}_{\theta_{\mathrm{b}}}^{-1})^{\intercal} \mathbf{M}'_{\mathrm{b}} \mathbf{R}_{\theta_{\mathrm{b}}}^{-1},\\
        & \mathbf{D}_{\mathrm{b}} = (\mathbf{R}_{\theta_{\mathrm{b}}}^{-1})^{\intercal} \mathbf{D}'_{\mathrm{b}} \mathbf{R}_{\theta_{\mathrm{b}}}^{-1},\\
        & \mathbf{C}_{\mathrm{b}} \dot{\bm{\eta}}_{\mathrm{b}}  := \frac{1}{2} \dot{\mathbf{M}}_{\mathrm{b}} \dot{\bm{\eta}}_{\mathrm{b}},
    \end{split}
\end{equation}
\noindent where $\dot{\mathbf{M}}_{\mathrm{b}} = \dot{\bm{\eta}}_{\mathrm{b}}^{\intercal} (\partial \mathbf{M}_{\mathrm{b}} / \partial \bm{\eta}_{\mathrm{b}})$ \cite{Fossen1995}. 
An explicit description of $\mathbf{D}_{\mathrm{b}}$ and $\mathbf{D}_{\mathrm{b}}$ is given by:
\begin{equation} \label{eq_buoy_inertia_matrix} 
    \mathbf{M}_{\mathrm{b}} = 
        \setlength{\arraycolsep}{2pt}
        \renewcommand{\arraystretch}{0.8}
        \begin{bmatrix} 
            M'_{\mathrm{b},11} c_{\theta_{\mathrm{b}}}^2 + M'_{\mathrm{b},22} s_{\theta_{\mathrm{b}}}^2   &  s_{2\theta_{\mathrm{b}}}(M'_{\mathrm{b},22}-M'_{\mathrm{b},11})/2  & 0 \\
            s_{2\theta_{\mathrm{b}}}(M'_{\mathrm{b},22}-M'_{\mathrm{b},11})/2  & M'_{\mathrm{b},11} s_{\theta_{\mathrm{b}}}^2 + M'_{\mathrm{b},22} c_{\theta_{\mathrm{b}}}^2 & 0 \\
            0 & 0 & M'_{\mathrm{b},33}\\    
        \end{bmatrix},
\end{equation}
\noindent where $M'_{\mathrm{b},ii}$, $i=1,2,3$ are elements of the buoy inertia matrix in $\mathcal{B}_{\mathrm{b}}$, $\mathbf{M}'_{\mathrm{b}}$. 
The buoy's damping matrix in the inertial frame $\mathcal{W}$ is defined as:
\begin{equation} \label{eq_buoy_damping_matrix} 
    \setlength{\arraycolsep}{2pt}
    \renewcommand{\arraystretch}{0.8}
    \mathbf{D}_{\mathrm{b}} =
        \begin{bmatrix}
            D'_{\mathrm{b},11} c_{\theta_{\mathrm{b}}}^2 + D'_{\mathrm{b},22} s_{\theta_{\mathrm{b}}}^2   &  s_{2\theta_{\mathrm{b}}}(D'_{\mathrm{b},22}-D'_{\mathrm{b},11})/2  & 0 \\
            s_{2\theta_{\mathrm{b}}}(D'_{\mathrm{b},22}-D'_{\mathrm{b},11})/2  & D'_{\mathrm{b},11} s_{\theta_{\mathrm{b}}}^2 + D'_{\mathrm{b},22} c_{\theta_{\mathrm{b}}}^2 & 0 \\
            0 & 0 & D'_{\mathrm{b},33}\\    
        \end{bmatrix},
\end{equation}
\noindent where $D'_{\mathrm{b},ii}$, $i=1,2,3$ are elements of the buoy damping matrix in $\mathcal{B}_{\mathrm{b}}$, $\mathbf{D}'_{\mathrm{b}}$.
We also let $M_{\mathrm{b},ij}$, $D_{\mathrm{b},ij}$, and $C_{\mathrm{b},ij}$, $i,j=1,2,3$ be elements of $\mathbf{M}_{\mathrm{b}}$, $\mathbf{D}_{\mathrm{b}}$, and $\mathbf{C}_{\mathrm{b}}$, respectively. 
%
\subsection{UAV's Dynamic Model} \label{subsec_UAV_model}
Referring to Assumption~\ref{assump_cable}, 
and applying Newton's second law of motion on the UAV quadrotor system in $\mathcal{W}$ with the state vector, $\bm{\eta}_{\mathrm{u}} = [x_{\mathrm{u}},z_{\mathrm{u}},\theta_{\mathrm{u}}]^{\intercal}$, yields:
\begin{equation} \label{eq_UAV_inertial_Newton}
    \mathbf{M}_{\mathrm{u}} \ddot{\bm{\eta}}_{\mathrm{u}} + \mathbf{D}_{\mathrm{u}} \tilde{\dot{\bm{\eta}}}_{\mathrm{u}} + 
    \mathbf{G}_{\mathrm{u}} = \bm{\tau}_{\mathrm{u}}, 
\end{equation}
\noindent where $\mathbf{M}_{\mathrm{u}} = diag(m_{\mathrm{u}},m_{\mathrm{u}},J_{\mathrm{u}}) \in \mathbb{R}_{> 0}^{3\times3}$ and $\mathbf{D}_{\mathrm{u}}=diag(D_{\mathrm{u},1},\,D_{\mathrm{u},2},\,D_{\mathrm{u},3}) \in \mathbb{R}_{\geq 0}^{3\times3}$
are the UAV's inertia and damping friction matrices, respectively;
the UAV's relative velocity vector is 
$\tilde{\bm{\dot{\eta}}}_{\mathrm{u}} = [\tilde{\dot{\eta}}_{\mathrm{u},1},\tilde{\dot{\eta}}_{\mathrm{u},2},\tilde{\dot{\eta}}_{\mathrm{u},3}]^{\intercal} = \bm{\dot{\eta}}_{\mathrm{u}} - [u_{\mathrm{wd}},0,0]^{\intercal}$, with $u_{\mathrm{wd}}$ being the horizontal wind velocity;
and $\mathbf{G}_{\mathrm{u}}$ and $\bm{\tau}_{\mathrm{u}} \in \mathbb{R}^3$ are vectors of the UAV's gravitational and other external forces and moments in $\mathcal{W}$, respectively, expressed as: 
\begin{equation} \label{eq_UAV_inertial_forces}
        \mathbf{G}_{\mathrm{u}}  = [0,m_{\mathrm{u}} \textsl{g},0]^{\intercal},\;
        \bm{\tau}_{\mathrm{u}}  = [u_1 s_{\theta_{\mathrm{u}}} - T c_{\alpha},u_1 c_{\theta_{\mathrm{u}}} - T s_{\alpha},u_2]^{\intercal}.
\end{equation}
The damping matrix element of interest, $D_{{\mathrm{u}},1}$, is approximated as:
\begin{equation} \label{eq_UAV_friction_coeff}
    D_{{\mathrm{u}},1} = C_{\mathrm{u},1} \mathrm{A}_{\mathrm{cs},1}^{\mathrm{u}} \frac{1}{2} \rho_{\mathrm{a}} |\tilde{\dot{\eta}}_{\mathrm{u},1}|,
\end{equation}
\noindent where $C_{\mathrm{u},1} \in \mathbb{R}_{>0}$ is a constant, $\mathrm{A}_{\mathrm{cs},1}^{\mathrm{u}} \in \mathbb{R}_{\geq 0}$ is the UAV's cross-sectional area across the $zy$-plane, and $\rho_{\mathrm{a}}$ is the air density.
\noindent For more details on the quadrotor UAV model, see \cite{Zhang2019}.
\par
\subsection{System Constraints} \label{subsec_system_constraints}
In order to fully define the marine locomotive problem as a coupled UAV$-$buoy system, specific constraints are required and are presented hereafter, with their violations depicted in Fig.~\ref{fig_constraints}.
%
\begin{figure}
\centerline{\includegraphics[width=3.45in]{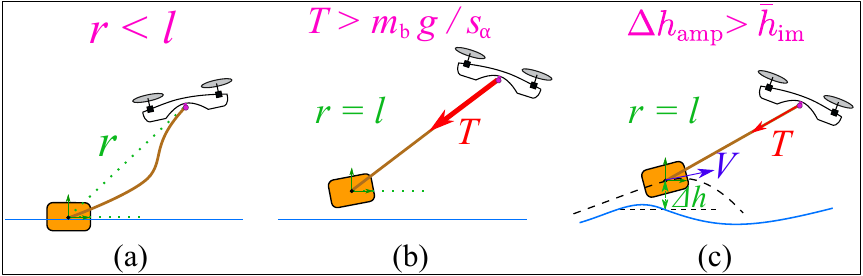}}
\caption{Depiction of the UAV$-$buoy system in violation of three constraints: (a) slack cable, (b) hanging buoy, and (c) `fly-over' phenomenon.}
\label{fig_constraints}
\end{figure}
\subsubsection{Taut Cable Constraint} \label{susubsec_positive_cable_tension}
This section introduces the resulting coupled dynamics of the UAV$-$buoy system, which is achieved when the tether links the two bodies and holds positive tension, i.e. with a taut-cable constraint that is opposite to what is shown in Fig.~\ref{fig_constraints}a. For this purpose, we let  $\mathcal{W}'=\{r',\alpha'\}$ be a rectilinear moving polar frame fixed to $\mathcal{O}_{\mathrm{b}}$, shown in Fig.~\ref{fig_Buoy_UAV_Annotations}; this frame does not rotate, and it is parallel to the inertial frame $\mathcal{W}$.
The position of the UAV in $\mathcal{W}$ with respect to $\mathcal{W'}$ is defined as: $\mathbf{r}=\mathbf{r}_{\mathrm{u}}-\mathbf{r}_{\mathrm{b}} \in \mathbb{R}^2$, 
and we let its coordinates in $\mathcal{W}'$ be $\mathbf{r}'=\{r,\alpha\}$, such that:
\begin{equation} \label{eq_UAV_polar_coordinates} 
        r = \|\mathbf{r}\|,\qquad
        \alpha = \text{atan2}(z_{\mathrm{u}}-z_{\mathrm{b}},x_{\mathrm{u}}-x_{\mathrm{b}}).
\end{equation}
\noindent We also let the rates vector, $\dot{\mathbf{r}}'$, be: 
\begin{equation} \label{eq_dot_r_in_W_prime} 
    \dot{\mathbf{r}}':=
    \begin{bmatrix}
        \dot{r}  \\
        r \dot{\alpha}  \\
    \end{bmatrix}
    =\mathbf{R}_{\alpha}^{\intercal}
    \begin{bmatrix}
        \dot{x}_{\mathrm{u}} - \dot{x}_{\mathrm{b}}  \\
        \dot{z}_{\mathrm{u}} - \dot{z}_{\mathrm{b}}  \\
    \end{bmatrix},
\end{equation}  
\noindent where $\mathbf{R}_{\alpha}^{\intercal}$ is the transformation matrix that rotates vectors in $\mathcal{W}$ to $\mathcal{W}'$, and we finally let the acceleration vector, $\ddot{\mathbf{r}}'$ be: 
\begin{equation} \label{eq_ddot_r_in_W_prime} 
    \ddot{\mathbf{r}}':=
    \begin{bmatrix}
        \ddot{r} -r\dot{\alpha}^2  \\
        r \ddot{\alpha} + 2 \dot{r} \dot{\alpha} \\
    \end{bmatrix}
    =\mathbf{R}_{\alpha}^{\intercal}
    \begin{bmatrix}
        \ddot{x}_{\mathrm{u}} - \ddot{x}_{\mathrm{b}}  \\
        \ddot{z}_{\mathrm{u}} - \ddot{z}_{\mathrm{b}}  \\
    \end{bmatrix}.
\end{equation}   
\begin{definition} \label{def_taut_cable}
    Based on Assumption~\ref{assump_cable}, the cable remains taut, i.e. it maintains tension, at time $t$ if $r(t) = l$. The taut-cable condition is expressed as:
\begin{equation} \label{eq_taut_cable_condition} 
    T > 0,
\end{equation}
\noindent under which the UAV$-$buoy system is labeled as `\textit{coupled}', otherwise it is labeled as `\textit{decoupled}'.
\end{definition}
\par
With Assumption \ref{assump_cable} and the taut-cable condition in (\ref{eq_taut_cable_condition}), we have $r=l$, and the polar coordinates of the UAV can be defined with respect to the buoy's center of mass, $\mathcal{O}_{\mathrm{b}}$, as: 
\begin{equation} \label{eq_UAV_pos_from_alpha} 
        x_{\mathrm{u}} = x_{\mathrm{b}} +l c_{\alpha},\qquad
        z_{\mathrm{u}} =z_{\mathrm{b}}+l s_{\alpha},
\end{equation}
and its velocity can be obtained as:
\begin{equation} \label{eq_UAV_vel_from_alpha} 
        \dot{x}_{\mathrm{u}} = \dot{x}_{\mathrm{b}}-(l s_{\alpha}) \dot{\alpha},\qquad
        \dot{z}_{\mathrm{u}} = \dot{z}_{\mathrm{b}}+(l c_{\alpha}) \dot{\alpha}.
\end{equation}
\par
\begin{lemma} \label{lemma_cable_tension}
    When the UAV$-$buoy system is coupled, the UAV's equations of motion in $\mathcal{W}'$ in polar coordinates notation are expressed as:
    \begin{equation} \label{eq_UAV_polar_frame} 
        \begin{split}
            - m_{\mathrm{u}} r \dot{\alpha}^2 & = m_{\mathrm{u}} ( - \ddot{x}_{\mathrm{b}} c_{\alpha} - \ddot{z}_{\mathrm{b}} s_{\alpha}) - m_{\mathrm{u}} \textsl{g} s_{\alpha} + u_1 s_{\alpha + \theta_{\mathrm{u}}} - T, \\
            m_{\mathrm{u}} r^2 \ddot{\alpha} & = m_{\mathrm{u}} r ( \ddot{x}_{\mathrm{b}} s_{\alpha} - \ddot{z}_{\mathrm{b}} c_{\alpha} ) - m_{\mathrm{u}} \textsl{g} r c_{\alpha} + r u_1 c_{\alpha + \theta_{\mathrm{u}}}.
        \end{split}
    \end{equation}
Also, let $V_{\mathrm{r}} := \tilde{\dot{\eta}}_{\mathrm{b},1} = V-U_{\mathrm{c}}-v_x^w$ represent the buoy$-$water relative surge velocity. If the condition in (\ref{eq_taut_cable_condition}) holds, the cable tension is expressed as:
    \begin{equation} \label{eq_cable_tension_estimation} T = 
    \begin{cases}
        \Big( \scalebox{0.80}{$ M_{\mathrm{b},11} \ddot{x}_{\mathrm{b}} + M_{\mathrm{b},12} \ddot{z}_{\mathrm{b}}
        + D_{\mathrm{b},11} V_{\mathrm{r}} $} & \\
        \;\;\; \scalebox{0.85}{$ +  D_{\mathrm{b},12} \tilde{\dot{\eta}}_{\mathrm{b},2} + C_{\mathrm{b},11} \dot{x}_{\mathrm{b}} + C_{\mathrm{b},12} \dot{z}_{\mathrm{b}} $} \Big) / c_{\alpha} ,
        & |\alpha - \frac{\pi}{2}| > \epsilon_{\alpha},\\
        %
        (u_1 c_{\theta_{\mathrm{u}}} - m_{\mathrm{u}} \textsl{g} - m_{\mathrm{u}} \ddot{z}_{\mathrm{u}}) / s_{\alpha}, & |\alpha - \frac{\pi}{2}| \leq \epsilon_{\alpha},
        \end{cases}
    \end{equation}  
    \noindent where $\epsilon_{\alpha} \in \mathbb{R}_{\geq0}$ is a constant that prevents singularity in a small region near $\alpha = \frac{\pi}{2}$.
\end{lemma}
\begin{proof}
    By differentiating $\mathbf{r}$ twice then multiplying it by $m_{\mathrm{u}}$, we get:
    \begin{equation} \label{eq_UAV_in_W_prime} 
        m_{\mathrm{u}} \ddot{\mathbf{r}}=m_{\mathrm{u}} \ddot{\mathbf{r}}_{\mathrm{u}} - m_{\mathrm{u}} \ddot{\mathbf{r}}_{\mathrm{b}}.
    \end{equation}
    Referring to Definition~\ref{def_taut_cable}, we must have $r(t)=l$ for the system to be coupled, that is $\dot{r}=\ddot{r}=0$. Thus, by referring to (\ref{eq_ddot_r_in_W_prime}), $\ddot{\mathbf{r}}$ reduces to the form:
    \begin{equation} \label{eq_ddot_r_in_W_prime_coupled} 
        \ddot{\mathbf{r}}=
        \mathbf{R}_{\alpha}
        \begin{bmatrix}
            - r \dot{\alpha}^2 \\
              r \ddot{\alpha}  \\
        \end{bmatrix}.
    \end{equation}
    Referring to (\ref{eq_UAV_inertial_Newton}), and
    combining it with (\ref{eq_UAV_in_W_prime}), then projecting along the radial and tangential directions of (\ref{eq_ddot_r_in_W_prime_coupled}) by means of $\mathbf{R}_{\alpha}^{\intercal}$, we can write the equations of motion of the UAV in $\mathcal{W}'$ in the polar coordinates notation as in (\ref{eq_UAV_polar_frame}).
    \par
The cable tension can be determined from the first row of the buoy dynamics in (\ref{eq_buoy_inertial_Newton}), so that its expression is more relevant to the coupled UAV$-$buoy system since it shows a direct link with $V_{\mathrm{r}}$, which yields the first case of (\ref{eq_cable_tension_estimation}).
%
However, this form is not applicable near the vertical cable configuration ($\alpha=\pi/2$) due to singularity, thus the actual cable tension, $T$, is computed via (\ref{eq_UAV_inertial_Newton}), which yields the second case of (\ref{eq_cable_tension_estimation}).
\end{proof}
%
\subsubsection{No Buoy-Hanging Constraint} \label{subsubsec_no_buoy_hanging_constraint}
The buoy is required to remain at the water surface level at all times, that is, the UAV must not lift the buoy into the air by means of the cable tension alone, as shown in Fig.~\ref{fig_constraints}b. This constraint can be forced by limiting the allowed cable tension by the following inequality, deduced from (\ref{eq_buoy_inertial_Newton}) and (\ref{eq_buoy_inertial_forces}) as:
\begin{equation} \label{eq_no_buoy_hanging_constraint} 
    T < m_{\mathrm{b}} \textsl{g} / s_{\alpha}.
\end{equation}
As noted in Remark~\ref{rem_buoy_density}, the buoy floats by itself, which means that no minimum cable tension is required to maintain the buoy at the water surface.
\subsubsection{No `Fly-Over' Constraint} \label{subsubsec_no_fly_over_constraint}
The buoy is required to remain in contact with the water surface at all times, that is, the UAV must not force it to `fly-over' the waves, as in Fig.~\ref{fig_constraints}c, when it encounters them within a specific frequency range. This constraint is described as:
\begin{equation} \label{eq_water_contact_condition} 
    \curlyvee_{\mathrm{im}} > 0,
\end{equation}
\noindent which guarantees keeping the buoy partially immersed at all times. 
`Fly-over' is a phenomenon that marks the flight of a planing hull over the waves level, thus losing contact with the water surface \cite{Fridsma1969PlaningBoats}. This phenomenon appears when the wave encounter frequency is near the resonant frequency of the hull, and is related to its Froude number \cite{Kim2013PlanningHull}.
\par
To detect the occurrence of this phenomenon, the following analysis is presented.
If the discontinuity in the buoyant force is neglected, the buoy's heave dynamics can be simplified and expressed as a second-order transfer function with natural frequency, $\omega_{b}$, and damping ratio, $\mu_{\mathrm{b}}$, deduced from (\ref{eq_buoy_inertial_Newton}) as:
\begin{equation} \label{eq_buoy_nat_freq_and_damp} 
        \omega_{\mathrm{b}} = \frac{\rho_{\mathrm{w}} \textsl{g} \mathrm{A}_{\mathrm{cs},3}^{\mathrm{b}}}{m_{\mathrm{b}}+a_{33}}, \qquad       
        \mu_{\mathrm{b}} = \frac{D_{\mathrm{b},33}}{2\sqrt{(m_{\mathrm{b}}+a_{33}) \rho_{\mathrm{w}} \textsl{g} \mathrm{A}_{\mathrm{cs},3}^{\mathrm{b}}}},
\end{equation}
\noindent where $\mathrm{A}_{\mathrm{cs},3}^{\mathrm{b}}$ is the mean horizontal cross-sectional area of the buoy at the water surface level. 
We also define $\omega_{\mathrm{e},n}$, the wave encounter frequency for the $n^{th}$ wave component, as \cite{faltinsen1990}:
\begin{equation} \label{eq_wave_encounter_freq} 
    \omega_{\mathrm{e},n} = \omega_n - d_n \frac{\omega_n^2 V}{\textsl{g}}, \quad n \in S_n.
\end{equation}
The `fly-over' phenomenon occurs at an exciting frequency where the increase in oscillations amplitude due to dynamic magnification, $\Delta h_{\text{amp}}$, exceeds the mean immersed height of the buoy, $\bar{h}_{\mathrm{im}}$, that is:
\begin{equation} \label{eq_fly_over_constraint} 
    \Delta h_{\text{amp}}:= \sum_{n}^{N} A_n \Big(\frac{1}{\sqrt{\big(1 -\bar{\omega}_n^2 \big)^2+(2 \mu_{\mathrm{b}} \bar{\omega}_n)^2}} - 1 \Big) > \bar{h}_{\mathrm{im}},
\end{equation}
\noindent where $\bar{\omega}_n=\omega_{\mathrm{e},n}/\omega_{\mathrm{b}}$.
Further elaboration on the implications of this condition on the system modeling and performance requires knowledge of the buoy characteristics in terms of shape and weight, as well as wave characteristics in terms of height and wave length, which are presented in Section \ref{sec_simulation}.
\par
%
\subsection{The Tethered UAV-Buoy System Model}
The formulation of the tethered UAV$-$buoy system, in its coupled form, is obtained via the Euler-Lagrange formulation, while incorporating the results of Sections \ref{subsec_buoy_model}, \ref{subsec_UAV_model}, and \ref{subsec_system_constraints}.
The Lagrangian function is obtained from the kinetic ($\mathcal{K}(\bm{q},\dot{\bm{q}}) \in \mathbb{R}_{\geq0}$) and potential ($\mathcal{U}(\bm{q}) \in \mathbb{R}$) energies as $\mathcal{L}(\bm{q},\dot{\bm{q}})=\mathcal{K}(\bm{q},\dot{\bm{q}})+\mathcal{U}(\bm{q})$,
where $\bm{q}=[x_{\mathrm{b}},z_{\mathrm{b}},\alpha,\theta_{\mathrm{u}},\theta_{\mathrm{b}}]^{\intercal} \in \mathbb{R}^5$ is the generalized coordinates vector.
The motion equations of the UAV$-$buoy system can then be derived as:
\begin{equation} \label{eq_Euler_Lagrange} 
        \frac{d}{dt}\Big(\frac{\partial \mathcal{L}}{\partial \dot{\bm{q}}}\Big)-\frac{\partial \mathcal{L}}{\partial \bm{q}}+\frac{\partial \mathcal{P}}{\partial \dot{\bm{q}}}=\bm{\tau},
\end{equation}
\noindent where $\bm{\tau} \in \mathbb{R}^{5}$ is the external forces vector; $\mathcal{P} \in \mathbb{R}$ is a power function that captures dissipative forces, such that $\frac{\partial \mathcal{P}}{\partial \dot{\bm{q}}}:=\mathbf{D} \tilde{\dot{\bm{q}}}$, where $\mathbf{D}$ is the global damping matrix that can be formulated based on (\ref{eq_buoy_body_to_inertial_trans}) without including a wind-induced component per Assumption~\ref{assump_neglected_air_friction}; and $\tilde{\dot{\bm{q}}}$ is defined as:
\begin{equation} \label{eq_configuration_variables_global} 
    \tilde{\dot{\bm{q}}}=[V_{\mathrm{r}},\dot{z}_{\mathrm{b}}-v_z^w,\dot{\alpha},\dot{\theta}_{\mathrm{u}},\dot{\theta}_{\mathrm{b}}]^{\intercal}. 
\end{equation}
\par
To facilitate the derivation of the Euler-Lagrange equations, the kinetic energy of the system is expressed as the sum of that of the buoy and that of the UAV:
\begin{equation} \label{eq_kinetic_energy} 
    \mathcal{K} = \frac{1}{2} \dot{\bm{q}}^{\intercal} \mathbf{M} \dot{\bm{q}} := \frac{1}{2} \dot{\bm{\eta}}_{\mathrm{b}}^{\intercal} \mathbf{M}_{\mathrm{b}} \dot{\bm{\eta}}_{\mathrm{b}} + \frac{1}{2} \dot{\bm{\eta}}_{\mathrm{u}}^{\intercal} \mathbf{M}_{\mathrm{u}} \dot{\bm{\eta}}_{\mathrm{u}},
\end{equation}
\noindent where $\mathbf{M}$ is the global inertia matrix of the UAV$-$buoy system, which can be formulated by referring to (\ref{eq_UAV_vel_from_alpha}) and using the elements of $\mathbf{M}_{\mathrm{b}}$ and $\mathbf{M}_{\mathrm{u}}$.
The system's potential energy and external forces and moments vector can be formulated based on (\ref{eq_buoy_inertial_forces}) and (\ref{eq_UAV_inertial_forces}) as:
\begin{equation} \label{eq_Euler_Lagrange_potention_forces} 
    \begin{split}
        \mathcal{U} & = m_{\mathrm{u}} \textsl{g} (z_{\mathrm{b}}+l s_{\alpha}) + m_{\mathrm{b}} \textsl{g} \,z_{\mathrm{b}},\\
        \bm{\tau} & = [u_1 s_{\theta_{\mathrm{u}}}, u_1 c_{\theta_{\mathrm{u}}} + \rho_{\mathrm{w}} \textsl{g} \curlyvee_{\mathrm{im}}, u_1 l c_{\alpha+\theta_{\mathrm{u}}},u_2,f_{\mathrm{p}}s_{\theta_{\mathrm{u}}}]^{\intercal}.\\
    \end{split}
\end{equation}
\par
Finally, the following equations of motion that result from Euler-Lagrange formulation (\ref{eq_Euler_Lagrange}) are obtained:
\begin{equation} \label{eq_dynamic_model_compact} 
    \mathbf{M} \ddot{\bm{q}} + \mathbf{C}\dot{\bm{q}} + \mathbf{D}\tilde{\dot{\bm{q}}} + \mathbf{G} = \bm{\tau},
\end{equation}
\noindent where $\mathbf{C} \dot{\bm{q}}:= \frac{1}{2} \dot{\mathbf{M}} \dot{\bm{q}}$ is the Coriolis matrix with $\dot{\mathbf{M}} = \dot{\bm{q}}^{\intercal} \frac{\partial \mathbf{M}}{\partial \bm{q}}$, and
the global vector of gravity forces $\mathbf{G}$ is:
\begin{equation} \label{eq_potential_forces} 
    \mathbf{G}:=\frac{\partial \mathcal{U}}{\partial \bm{q}}=[0,(m_{\mathrm{b}}+m_{\mathrm{u}}) \textsl{g},m_{\mathrm{u}} \,\textsl{g}\, l c_{\alpha},0,0]^{\intercal}.
\end{equation}
\begin{assumption} \label{assump_stable_buoy_rot_dyn}
    The buoy's pitch dynamics are damped and stable, that is: $D_{\mathrm{b},33} \neq 0$ and $f_{\mathrm{p}}<0$. As a result, the buoy is assumed to remain tangent to the water surface at all times.
\end{assumption} \par
With Assumption \ref{assump_stable_buoy_rot_dyn} and the dominance of waves with relatively long wave period and moderate wave height, the time derivative of the buoy pitch angle, $\dot{\theta}_{\mathrm{b}}$, is small and thus its effect can be neglected in $\dot{\mathbf{M}}$, which yields a Coriolis matrix that is a function of $\alpha$ only.
\par
With constraints (\ref{eq_taut_cable_condition}) and (\ref{eq_water_contact_condition}) satisfied, the dynamic model equations in the coupled form are given by:
\begin{subequations}
\label{eq_dynamic_model_expended}
\begin{align}
     \begin{split}
        & (M_{\mathrm{b},11}+m_{\mathrm{u}})\ddot{x}_{\mathrm{b}} + M_{\mathrm{b},12}\ddot{z}_{\mathrm{b}}  +  D_{\mathrm{b},11} V_{\mathrm{r}} + D_{\mathrm{b},12} \tilde{\dot{z}}_{\mathrm{b}} \\
        & \qquad - m_{\mathrm{u}} l(c_{\alpha}\dot{\alpha}^2 + s_{\alpha}\ddot{\alpha}) = u_1 s_{\theta_{\mathrm{u}}}, 
    \end{split} 
   \label{eq_dynamic_model_expended_a} \\
    %
    \begin{split}
        & (M_{\mathrm{b},22}+m_{\mathrm{u}})\ddot{z}_{\mathrm{b}} + M_{\mathrm{b},21}\ddot{x}_{\mathrm{b}} - m_{\mathrm{u}} l(s_{\alpha}\dot{\alpha}^2 - c_{\alpha}\ddot{\alpha}) \\
        & + D_{\mathrm{b},22} \tilde{\dot{z}}_{\mathrm{b}} + D_{\mathrm{b},21} V_{\mathrm{r}} + (m_{\mathrm{b}}+m_{\mathrm{u}})\textsl{g} = u_1 c_{\theta_{\mathrm{u}}} \\
        & + (\rho_{\mathrm{w}} \curlyvee_{\mathrm{im}}) \textsl{g},  
    \end{split}   \label{eq_dynamic_model_expended_b} \\
    & m_{\mathrm{u}} l^2 \ddot{\alpha} + m_{\mathrm{u}} l( -s_{\alpha}\ddot{x}_{\mathrm{b}} + c_{\alpha}\ddot{z}_{\mathrm{b}}) + m_{\mathrm{u}} \textsl{g}(l c_{\alpha}) \\
    & = u_1 l c_{\alpha+\theta_{\mathrm{u}}}, 
\label{eq_dynamic_model_expended_c} \\
    & J_{\mathrm{u}} \ddot{\theta}_{\mathrm{u}} = u_2, \label{eq_dynamic_model_expended_d} \\
    & M_{\mathrm{b},33} \ddot{\theta}_{\mathrm{b}} +D_{\mathrm{b},33} \dot{\theta}_{\mathrm{b}} = f_{\mathrm{p}} s_{\theta_{\mathrm{b}}}, \label{eq_dynamic_model_expended_e}
\end{align} 
\end{subequations}
\noindent where $\tilde{\dot{z}}_{\mathrm{b}}:=\tilde{\dot{\eta}}_{\mathrm{b},2}$.
The UAV's position and velocity vectors can then be obtained from (\ref{eq_UAV_pos_from_alpha}) and (\ref{eq_UAV_vel_from_alpha}), respectively.
\begin{remark} \label{remark_unsatisfied_conditions}
    If the taut-cable constraint (\ref{eq_taut_cable_condition}) is not satisfied, the system in (\ref{eq_dynamic_model_expended}) decouples into (\ref{eq_buoy_inertial_Newton}) and (\ref{eq_UAV_inertial_Newton}) with $T=0$, and the polar states $\mathbf{r}'$, $\dot{\mathbf{r}}'$, and $\ddot{\mathbf{r}}'$ are calculated from (\ref{eq_UAV_polar_coordinates}), (\ref{eq_dot_r_in_W_prime}), and (\ref{eq_ddot_r_in_W_prime}), respectively. 
    On the other hand, if the fly-over constraint (\ref{eq_water_contact_condition}) is not satisfied, the buoy's inertia matrix in (\ref{eq_buoy_inertia_matrix}) reduces to $\mathbf{D}_{\mathrm{b}}=diag(m_{\mathrm{b}},m_{\mathrm{b}},J_{\mathrm{b}})$, the buoy's damping matrix $\mathbf{D}_{\mathrm{b}}$ in (\ref{eq_buoy_damping_matrix}) reduces to a null matrix, and $f_{\mathrm{p}}$ becomes zero.
\end{remark}
\par

\section{Control System Design} \label{sec_controller_design} 
The control system design problem is defined as manipulating the surge velocity of the buoy, $V$, to track a desired reference and to maintain the UAV's elevation, $z_{\mathrm{u}}$, at a constant level, while ensuring that the dynamics of the UAV$-$buoy system remain stable and contact between the buoy and water is maintained.
%
\par
%
\subsection{Attainable Setpoints} 
%
The control objective is to attain a steady-state mean velocity of the buoy, ($\bar{V}$), and mean UAV's elevation, ($\bar{z}_{\mathrm{u}}$), such that $lim_{t \rightarrow \infty} (\frac{1}{t} \int {z_{\mathrm{u}}(t) dt}, \frac{1}{t} \int{V(t) dt})=(\bar{z}_{\mathrm{u}},\bar{V})$.
Next, we seek to find the set of other system states, namely, $\bar{\theta}_{\mathrm{u}}$ and $\bar{\curlyvee}_{\mathrm{im}}$, and control inputs, $\bar{u}_1$ and $\bar{u}_2$, that will achieve the control objective.
Other nonzero mean system variables in a steady-state surge motion are: $\bar{T}$, $\bar{D}_{\mathrm{b},11}$, and $\bar{D}_{\mathrm{b},21}$. Note that the bar sign ($\bar{\bigcdot}$) refers to the mean values of the variables at equilibrium, $\bar{\curlyvee}_{\mathrm{im}}$ implicitly represents $\bar{z}_\mathrm{b}$, and the buoy's pitch angle, $\theta_{\mathrm{b}}$, is not considered in the setpoint analysis per Assumption~\ref{assump_stable_buoy_rot_dyn}.
\par
\begin{definition} \label{def_control_problem}
    Under specific sea conditions, namely $U_{\mathrm{c}}$ and $\zeta(A_n,\omega_n)$ with $n \in S_n$, 
    and certain safety margins $\epsilon_{T}(U_{\mathrm{c}},\zeta) \geq 0$ for the cable's tension, which guarantees the \textit{coupled} state of the system; and $\epsilon_{\curlyvee} \in (0,1)$ for the buoy's immersed volume, to ensure a minimum buoy immersion that is suitable to the desired system application;
    the set of admissible configurations consists of the equilibrium points $(\bar{V},\bar{\curlyvee}_{\mathrm{im}},\bar{\alpha},\bar{\theta}_{\mathrm{u}})$, such that:
    \begin{equation}     \label{eq_safety_margins} 
            \bar{T}:=\bar{T}(\bar{V}_{\mathrm{r}},\bar{\alpha}) > \epsilon_{T},\qquad
            \bar{\curlyvee}_{\mathrm{im}}:=\bar{\curlyvee}_{\mathrm{im}}(\bar{V}_{\mathrm{r}},\bar{\alpha}) > \epsilon_{\curlyvee} \curlyvee_{\mathrm{b}},
    \end{equation}   
    \noindent where $\bar{V}_{\mathrm{r}}= \bar{V} - U_{\mathrm{c}}$. 
\end{definition}
\begin{assumption} \label{assump_equilibrium_NO_wave}
    The equilibrium state is analyzed under the no-wave condition: $A_n=0$ with $n \in S_n$, that is, $v_x^w = v_z^w = 0$.
\end{assumption}
\begin{theorem}
    Consider the system described in (\ref{eq_dynamic_model_expended}), subject to constraints (\ref{eq_taut_cable_condition}) and (\ref{eq_water_contact_condition}), and the margins specified in (\ref{eq_safety_margins}); by Assumptions \ref{assump_stable_buoy_rot_dyn} and \ref{assump_equilibrium_NO_wave}, the set of attainable equilibrium states is the union of all $(\bar{V},\bar{\curlyvee}_{\mathrm{im}},\bar{\alpha},\bar{\theta}_{\mathrm{u}})$ that satisfy: 
    \begin{equation} \label{eq_theta_bar} 
        \bar{\theta}_{\mathrm{u}}(\bar{V}_{\mathrm{r}},\bar{\alpha}) = \mathrm{atan} \Big(\frac{\bar{D}_{\mathrm{b},11}  \bar{V}_{\mathrm{r}} c_{\bar{\alpha}}}{m_{\mathrm{u}} \,\textsl{g}\, c_{\bar{\alpha}} + \bar{D}_{\mathrm{b},11} \bar{V}_{\mathrm{r}} s_{\bar{\alpha}}} \Big),
    \end{equation}
    %
    \noindent and the steady-state thrust value and immersed volume are calculated as:
    \begin{subequations}
    \label{eq_u1_bar_Vol_bar}
        \begin{align}
            & \bar{u}_1 = 
            \begin{cases}
                \mathrm{any} \: \mathbb{R}_{>0}, &  \quad \mathrm{if} \; \bar{\alpha} = \frac{\pi}{2} \\
                \bar{D}_{\mathrm{b},11} \bar{V}_{\mathrm{r}} / s_{\bar{\theta}_{\mathrm{u}}}, & \quad  \mathrm{otherwise},
            \end{cases}
        \label{eq_u1_bar_Vol_bar_a}\\
            & \bar{\curlyvee}_{\mathrm{im}} =
            \begin{cases}
                \frac{m_{\mathrm{b}} + m_{\mathrm{u}}} {\rho_{\mathrm{w}}} - \frac{\bar{u}_1}{\rho_\mathrm{w} \textsl{g}}, & \quad \mathrm{if} \;  \bar{\alpha} = \frac{\pi}{2} \\
                \frac{m_{\mathrm{b}}}{\rho_{\mathrm{w}}} - \frac{ \bar{V}_{\mathrm{r}} } {\rho_{\mathrm{w}} \textsl{g}} \Big(  \bar{D}_{\mathrm{b},11} t_{\bar{\alpha}} - \bar{D}_{\mathrm{b},21}  \Big), & \quad \mathrm{otherwise}.
            \end{cases} \label{eq_u1_bar_Vol_bar_b}
        \end{align}   
    \end{subequations}               
Given $\bar{\curlyvee}_{\mathrm{im}}$, we can solve for $\bar{z}_{\mathrm{b}}$ per specific buoy geometry. 
In addition, the cable tension at equilibrium is a function of $\bar{V}_{\mathrm{r}}$ and $\bar{\alpha}$, and expressed as:
    \begin{equation} \label{eq_cable_tension_equilibrium} 
        \bar{T} = \begin{cases}
            \mathrm{any} \: \mathbb{R}_{>0}, & \quad \mathrm{if} \; \bar{\alpha} = \frac{\pi}{2} \\
            \bar{D}_{\mathrm{b},11} \bar{V}_{\mathrm{r}} / c_{\bar{\alpha}}, & \quad \mathrm{otherwise}. \\
    \end{cases}
    \end{equation}
\end{theorem}
\begin{proof}
    The dynamic equilibrium of system (\ref{eq_dynamic_model_expended}) is attained when $\ddot{x}_{\mathrm{b}}=\ddot{z}_{\mathrm{b}}=\ddot{\theta}_{\mathrm{u}}=\dot{\theta}_{\mathrm{u}}=\ddot{\theta}_{\mathrm{b}}=\dot{\theta}_{\mathrm{b}}=\ddot{\alpha}=\dot{\alpha}=0$, and since we are considering surface motion of the buoy along with Assumption \ref{assump_equilibrium_NO_wave}, we additionally have $\dot{z}_{\mathrm{b}}=0$. Thus, we conclude that $u_1 = \bar{u}_1$, and $u_2 = \bar{u}_2 :=0$, and by substituting in (\ref{eq_dynamic_model_expended}), we get:
    \begin{subequations}
    \label{eq_steady_state_surge}
        \begin{align}
            & \bar{D}_{\mathrm{b},11} \bar{V}_{\mathrm{r}} - \bar{u}_1 s_{\bar{\theta}_{\mathrm{u}}} = 0, 
            \label{eq_steady_state_surge_a}\\
            & \bar{D}_{\mathrm{b},21} \bar{V}_{\mathrm{r}} + (m_{\mathrm{b}}+m_{\mathrm{u}})\textsl{g} - \bar{u}_1 c_{\bar{\theta}_{\mathrm{u}}} - (\rho_{\mathrm{w}} \bar{\curlyvee}_{\mathrm{im}}) \textsl{g} = 0, \label{eq_steady_state_surge_b}\\
            & m_{\mathrm{u}} \textsl{g} c_{\bar{\alpha}} - \bar{u}_1 c_{\bar{\alpha} + \bar{\theta}_{\mathrm{u}}} = 0. \label{eq_steady_state_surge_c}
        \end{align}
    \end{subequations}
    \noindent $\bar{\theta}_{\mathrm{u}} (\bar{V}_{\mathrm{r}},\bar{\alpha})$ in (\ref{eq_theta_bar}) is obtained by rearranging and dividing (\ref{eq_steady_state_surge_a}) by (\ref{eq_steady_state_surge_c}); in the case when $\bar{\alpha} \neq \frac{\pi}{2}$, $\bar{u}_1$ can be subsequently obtained from (\ref{eq_steady_state_surge_a}). As for $\bar{\curlyvee}_{\mathrm{im}}$, it can be obtained after substituting for $\bar{u}_1$ from (\ref{eq_u1_bar_Vol_bar_a}) and for $\bar{\theta}_{\mathrm{u}}$ from (\ref{eq_theta_bar}) in (\ref{eq_steady_state_surge_b}). Finally, the cable tension at equilibrium can be obtained from (\ref{eq_cable_tension_estimation}) by canceling the zero-valued states.
    Note that when $\bar{\alpha} = \frac{\pi}{2}$, the system of equations (\ref{eq_steady_state_surge}) has a solution only if $\bar{V}_{\mathrm{r}}=\bar{\theta}_{\mathrm{u}}=0$, while $\bar{u}_1$ can be any $\mathbb{R}_{>0}$ that respects the system constraints, and can be chosen to manipulate $\bar{\curlyvee}_{\mathrm{im}}$ based on the first case of (\ref{eq_u1_bar_Vol_bar_b}).
\end{proof}
\par
Now we seek to define the set of possible attainable steady-state velocities, $S_{\bar{V}}$, under Assumption~\ref{assump_equilibrium_NO_wave}, that satisfy the safety margins specified in (\ref{eq_safety_margins}).
The cable tension at equilibrium, $\bar{T}$, can be obtained from (\ref{eq_cable_tension_equilibrium}); then we can determine the minimum absolute surface velocity, $\bar{V}$, under a specific sea state, i.e. current and waves, that guarantees the taut-cable condition.
In addition, the maximum absolute limit of $S_{\bar{V}}$ is attained from (\ref{eq_safety_margins}) and (\ref{eq_u1_bar_Vol_bar_b}).
Finally, we get $S_{\bar{V}} = S_{\bar{V}\mathrm{n}} \cup S_{\bar{V}\mathrm{p}} $, such that:
\begin{equation} \label{eq_valid_vel_set} 
    \begin{split}
        S_{\bar{V}\mathrm{p}} & = \Big(\frac{\epsilon_{T} c_{\bar{\alpha}}}{ \bar{D}_{\mathrm{b},11} }  + U_{\mathrm{c}}, 
        \frac{(m_{\mathrm{b}}+m_{\mathrm{u}} - \epsilon_{\curlyvee} \rho_{\mathrm{w}}) \textsl{g} \, t_{\bar{\theta}_{\mathrm{u}}}}{\bar{D}_{\mathrm{b},11}} + U_{\mathrm{c}} \Big), \\
        & \qquad\qquad\qquad\qquad\qquad\qquad\qquad \text{if} \; \bar{\alpha} \leq \frac{\pi}{2} \\
        S_{\bar{V}\mathrm{n}} & = \Big(
        \frac{(m_{\mathrm{b}}+m_{\mathrm{u}} - \epsilon_{\curlyvee} \rho_{\mathrm{w}}) \textsl{g} \, t_{\bar{\theta}_{\mathrm{u}}}}{\bar{D}_{\mathrm{b},11}} + U_{\mathrm{c}},
        \frac{\epsilon_{T} c_{\bar{\alpha}}}{ \bar{D}_{\mathrm{b},11} }  + U_{\mathrm{c}} \Big), \\
        & \qquad\qquad\qquad\qquad\qquad\qquad\qquad \text{if} \; \bar{\alpha} >  \frac{\pi}{2}. \\
    \end{split}
\end{equation} 
\begin{remark} \label{rem_other_velocity_bounds}
    In practice, the maximum attainable absolute velocity can be limited by the UAV's maximum thrust, which can be derived from (\ref{eq_theta_bar}) and (\ref{eq_u1_bar_Vol_bar_a}), and the tether's yield strength. It is also noted that the motion across waves of various characteristics may alter the velocity bounds as to be discussed in Section~\ref{subsec_velocity_bounds}. In addition, the violation of the buoy's velocity upper bound can be alternatively prevented by referring to constraint (\ref{eq_no_buoy_hanging_constraint}) and the cable tension calculation in (\ref{eq_cable_tension_estimation}), and limiting the UAV's maximum thrust such that:
    \begin{equation} \label{eq_uav_thrust_constraint} 
        u_{1} < \frac{m_{\mathrm{b}}(1-\epsilon_{m}) \textsl{g}}{t_{\alpha} s_{\theta_{\mathrm{u}}}},
    \end{equation} 
    \noindent where $\epsilon_{m} \in [0,1)$ is a safety margin that represents a fraction of the buoy's mass, and accounts for the unmodeled dynamic forces affecting the buoy's heave motion that might violate the system constraints in (\ref{eq_no_buoy_hanging_constraint}) and (\ref{eq_water_contact_condition}).
\end{remark}
%
\subsection{Operational Modes and State Machine}
To achieve acceleration and deceleration motions,
the UAV$-$buoy system is required to manipulate the cable tension, switch between coupled and decoupled states, and achieve bidirectional velocity control; hence, the UAV must change its positioning with respect to the buoy back and forth.
Thus, the locomotive UAV control system is to be designed to operate in both position control and velocity control modes, which necessitates the use of a state machine to achieve an autonomous performance of the UAV$-$buoy system.
Note that a cable can only transmit tensile forces, thus allowing only pulling actions.
\par
Next, we provide the required definitions to describe the system states, present a complete cycle of the system's operational states to achieve the control objectives, and we introduce a state machine that allows the execution of appropriate commands. 
\subsubsection{Operational Modes} 
\begin{definition} \label{def_operational_modes}
The UAV's location with respect to the buoy is assigned one of the following two configurations:
    \begin{itemize}
        \item We call `\textit{front}' the configuration at which the UAV is positioned to the front of the buoy, i.e. $\alpha \in (0,\frac{\pi}{2})$.
        \item We call `\textit{rear}' the configuration at which the UAV is positioned to the rear of the buoy, i.e. $\alpha \in (\frac{\pi}{2},\pi)$.
    \end{itemize}
\end{definition}
We let $\bar{r}$ be the reference radial position of the UAV with respect to the buoy. The UAV$-$buoy system can be in one out of four operational modes shown in Fig.~\ref{fig_all_flying_modes}:
    \begin{enumerate}[label=(\alph*)]
        \item We call `\textit{free}' the mode during which the UAV is allowed to move freely around the buoy, while $r<l$. 
        \item We call `\textit{ready to pull}' the mode during which the UAV is commanded to maintain a specific elevation ($\bar{z}_{\mathrm{u}}$), and a reference standby radius, $r_{\mathrm{sb}}$, which is slightly less than the cable length $l$ to consume any cable slack. The elevation angle is $\alpha_0$ if the configuration is `front' and ($\pi - \alpha_0$) if the configuration is `rear'.
        \item We call `\textit{repositioning}' the mode during which the UAV moves from one side of the buoy to the other (fore/aft), travels a total arc of ($\pi-2 \alpha_0$), while maintaining a constant reference radius with respect to the buoy, $\bar{r}=r_{\mathrm{sb}}$, until it returns to the initial elevation, $\bar{z}_{\mathrm{u}}$.
        \item We call `\textit{pulling}' the mode during which the UAV is performing a pulling action on the buoy with a reference elevation, $\bar{z}_{\mathrm{u}}$, and radius, $r=l$. The resulting elevation angle is $\alpha_0'$ if the configuration is `front' and ($\pi - \alpha_0'$) if the configuration is `rear'.
    \end{enumerate}
\par
%
\begin{figure}
\centerline{\includegraphics[width=3.45in]{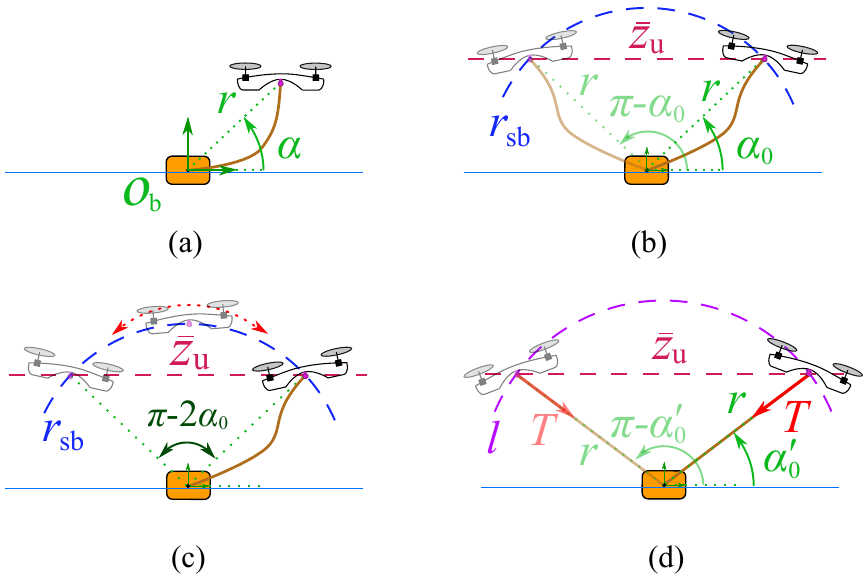}} 
\caption{UAV$-$buoy system operational states in the locomotion task: (a) free UAV motion around the buoy within the cable limit, used in initializing the system, (b) ready to pull forward (or backward), the UAV is in the right position to generate tension in the cable when asked to do so, (c) switching UAV's positioning between front and rear, while following the trajectory marked in dashed blue to avoid cable entanglement, and (d) coupled and pulling forward (or backward) to manipulate the buoy surge velocity.}
\label{fig_all_flying_modes}
\end{figure}
%
%
\par
\subsubsection{State Machine} \label{sub_sec_state_machine}
The UAV$-$buoy system is supervised by a state machine that governs switching between different control modes and commanded actions.
Fig.~\ref{fig_state_switching_diagram} illustrates a typical velocity profile, with a hypothetical tracking performance, and threshold lines that govern the state machine actions to showcase the state switching mechanism.
Let the threshold levels be denoted by $\mathrm{p}_1$ and $\mathrm{p}_2$ for the top two lines, and $\mathrm{n}_1$ and $\mathrm{n}_2$ for the bottom two lines. 
The first and second velocity error thresholds are denoted by $\epsilon_{\mathrm{th}1}$ and $\epsilon_{\mathrm{th}2}$ respectively.
\par
The proposed state machine benefits from the threshold velocity lines to choose the suitable mode of action as described in Definition~\ref{def_operational_modes}, in a way that respects the system dynamics and insures the system safety \cite{Tognon2020AerialRobotsTethers}, with a pseudo-code provided in Algorithm~\ref{alg_state_machine}.
\par
\begin{figure}
\centerline{\includegraphics[width=3.45 in]{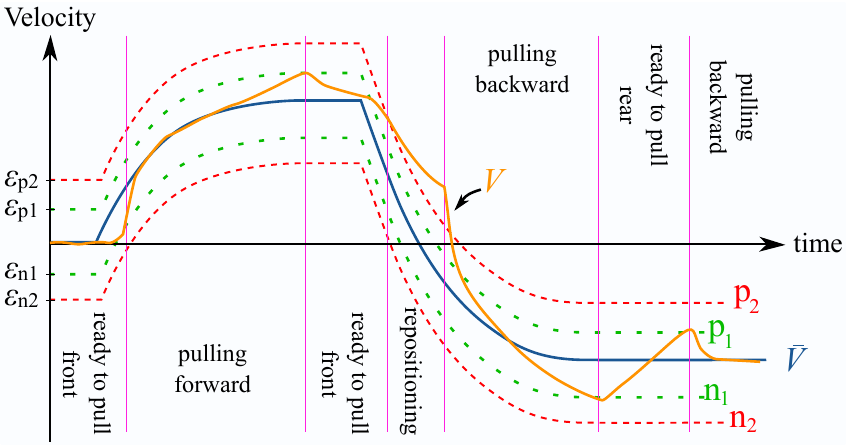}} 
\caption{Demonstrative diagram showing the modes' transition behavior and the buoy's velocity tracking performance during a theoretical scenario. The doted boundary lines govern the actions of the state machine.}
\label{fig_state_switching_diagram}
\end{figure}
%
\begin{figure}
 \begin{algorithm}[H] 
 \caption{State machine for the locomotive UAV's control system}
 \begin{algorithmic}[1] \label{alg_state_machine}
 \renewcommand{\algorithmicrequire}{\textbf{Input:}}
 \renewcommand{\algorithmicensure}{\textbf{Output:}}
 \REQUIRE $V$, $\bar{V}$, \textit{configuration}
 \ENSURE  \textit{MODE}
  \\ \textit{Initialization} :
  \\ \textit{MODE} $\Leftarrow$ `free'
  \\ \textit{LOOP Process}
  \IF {\quad\;\;\,($V < \bar{V} - \epsilon_{\mathrm{th1}}$) \AND \textit{configuration} $==$ `front'}
  \STATE \textit{MODE} $\Leftarrow$ `pulling'
  \ELSIF{($V > \bar{V} + \epsilon_{\mathrm{th1}}$) \AND \textit{configuration} $==$ `front'}
  \STATE \textit{MODE} $\Leftarrow$ `ready to pull'
  \ELSIF{($V > \bar{V} + \epsilon_{\mathrm{th2}}$) \AND \textit{configuration} $==$ `front'}
  \STATE \textit{MODE} $\Leftarrow$ `repositioning'
  \ELSIF {($V > \bar{V} + \epsilon_{\mathrm{th1}}$) \AND \textit{configuration} $==$ `rear'\,}
  \STATE \textit{MODE} $\Leftarrow$ `pulling'
  \ELSIF{($V < \bar{V} - \epsilon_{\mathrm{th1}}$) \AND \textit{configuration} $==$ `rear'\,}
  \STATE \textit{MODE} $\Leftarrow$ `ready to pull'
  \ELSIF{($V < \bar{V} - \epsilon_{\mathrm{th2}}$) \AND \textit{configuration} $==$ `rear'\,}
  \STATE \textit{MODE} $\Leftarrow$ `repositioning'
  \ENDIF
 \RETURN \textit{MODE}
 \end{algorithmic}
 \end{algorithm}
\end{figure}
%
\subsection{Controller Design} 
The control system of the tethered UAV$-$buoy system consists of an outer-loop and an inner-loop controller in a cascaded structure.
The outer-loop controller has two functions: 1) it controls the UAV's relative position when the system mode is `free', `ready to pull', or `repositioning' by controlling $r$ and $\alpha$, with setpoint $(\bar{z}_{\mathrm{u}0},\bar{r}_0)$; and 2) it controls the buoy's velocity when the system mode is `pulling', by regulating the elevation angle, $\alpha$, and the cable tension, $T$, which are two flat outputs of the coupled system \cite{Tognon2017}, with setpoint $(\bar{z}_{\mathrm{u}0},\bar{V}_0)$.
On the other hand, the inner-loop controller controls and stabilizes the UAV's pitch angle, $\theta_{\mathrm{u}}$. 
The proposed controller (SVCS) incorporates the state machine in Section~\ref{sub_sec_state_machine}, and it is designed based on polar coordinates. The SVCS architecture is presented in Fig.~\ref{fig_Controller_Diagram}, which can be summarized as follows: 
    \begin{itemize}
        \item A setpoint is defined and the state machine returns the system mode. 
        \item A preprocessing unit generates 1) an elevation angle $\bar{\alpha}$ that accounts for the actual buoy elevation variation, 2) a reference radial distance $\bar{r}$, 3) a smoothed reference velocity $\bar{V}$, and 4) an estimate for the required cable tension, $\hat{T}_{\mathrm{c}}$, to compensate for water drag, if applicable.
        \item The outer-loop controller generates radial and tangential components of the desired force that is needed for cable tension control, $u_r^{\mathrm{v}}$, in case of velocity control or simply radial force, $u_r^{\mathrm{p}}$, in case of position control (radial), and the elevation angle control, $u_{\alpha}$ (tangential).
        Note that the switching between $u_r^{\mathrm{p}}$ and $u_r^{\mathrm{v}}$ is governed by the operational mode of the system such that:
        \begin{equation} \label{eq_u_switching}
            u_r = \big(1 - \mathrm{f_{pl}} \mathrm{H}(s)\big) u_r^{\mathrm{p}} + \mathrm{f_{pl}} u_r^{\mathrm{v}},
        \end{equation}
        \noindent where $\mathrm{H}(s)$ is the transfer function of a low-pass filter, and $\mathrm{f_{pl}} \in \{0,1\}$ is the `pulling' mode flag.
        \item The outer-loop controller outputs are decoupled into a command total thrust, $u_1$, and a desired pitch angle, $\theta_{\mathrm{u,c}}$. 
        \item Finally, the inner-loop attitude controller stabilizes the pitch angle of the UAV and produces the moment command input, $u_2$. 
        \end{itemize}
\par
%
\begin{figure}
\centerline{\includegraphics[width=3.45in]{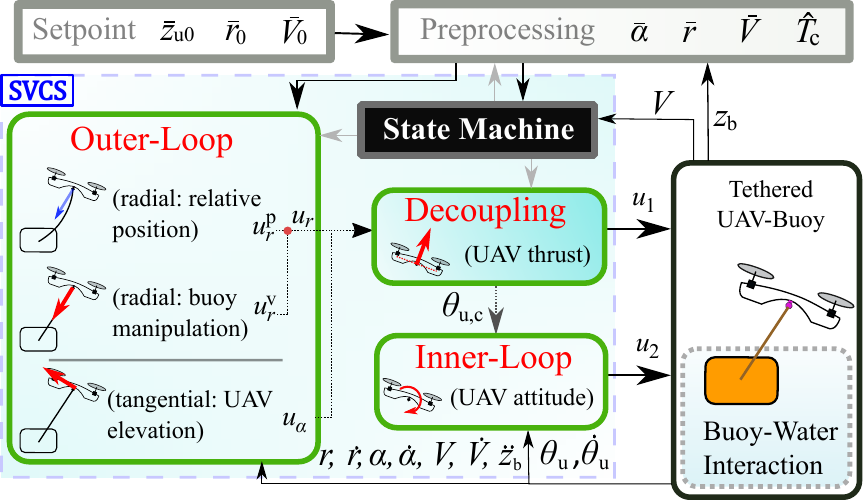}}
\caption{Architecture of the Surge Velocity Control System (SVCS) for the tethered UAV$-$buoy system.}
\label{fig_Controller_Diagram}
\end{figure}
%
\par
\subsubsection{Reference Signals and Velocity Setpoint}
It is desired for the UAV to maintain the same altitude during operation in order to respect aviation safety margins and save energy by reducing unnecessary vertical motion. The cable length and nominal elevation angle are chosen accordingly. 
However, due to the vertical oscillatory motion of the buoy in accordance with the encountered waves, we must actively provide the controller with suitable elevation angle, $\bar{\alpha}$, to hold the desired UAV elevation, which is computed as:
\begin{equation} \label{eq_alpha_desired}
    \bar{\alpha}=\text{asin}\big((\bar{z}_{\mathrm{u}}-z_{\mathrm{b}})/\bar{r} \big).
\end{equation}
\noindent Note that the preprocessing unit outputs the supplementary angle of $\bar{\alpha}$, i.e. $\bar{\alpha} \Leftarrow \pi - \bar{\alpha}$, if the system configuration is `\textit{rear}'.
Furthermore, the velocity setpoint, $\bar{V}_0$, and the radial position, $\bar{r}_0$, are smoothed by second-order and fourth-order low-pass filters, respectively, in order to respect the system dynamics in terms of buoy$-$water friction and the UAV's maximum thrust, thus preventing excessive coupling and decoupling of the system \cite{Tognon2017}.
Finally, in order to improve the performance of the state machine, $\bar{V}$ is sent to the controller only when the UAV is ready to enter the `pulling' mode.
\par
\subsubsection{UAV-Buoy Relative Position Control Law} \label{sec_rel_pos_control_law}
Consider the UAV$-$buoy's relative position dynamics in (\ref{eq_UAV_polar_frame}) for the generic case, i.e. nonzero tension, and the UAV's attitude dynamics in (\ref{eq_UAV_inertial_Newton}), with states vector $\bm{X}_1=[r,\, \alpha,\, \theta_\mathrm{u}]^{\intercal}$ and $\bm{X}_2=[\dot{r},\, \dot{\alpha},\, \dot{\theta}_\mathrm{u}]^{\intercal}$, and control input vector $\bm{U}=[u_r^{\mathrm{p}},\, u_{\alpha},\, u_2]^{\intercal}$, such that:
\begin{equation} \label{eq_control_inputs_transform}
        u_r = u_1 s_{\alpha + \theta_{\mathrm{u}}}, \qquad
        u_{\alpha} = u_1 c_{\alpha + \theta_{\mathrm{u}}},
\end{equation}
\noindent and subject to unknown external disturbances like wind gusts, gravity waves, and water currents. Note that the relation between $u_{r}^{\mathrm{p}}$ and $u_r$ was given in (\ref{eq_u_switching}). When represented as a nonlinear second-order time-varying system, the state space form is described as:
\begin{equation} \label{eq_ss_rat_model}
\begin{split}
    \dot{\bm{X}}_1 & = \bm{X}_2, \\
    \dot{\bm{X}}_2 & = \bm{H} + \bm{\Phi} \bm{\Theta} + \bm{b}\bm{U} + \bm{\delta},
\end{split}
\end{equation}
\noindent where $\bm{b} = diag(m_{\mathrm{u}},\,m_{\mathrm{u}} r,\,J_{\mathrm{u}})^{-1}$ is the input-multiplied vector, $\bm{\Phi}=[1/{(m_{\mathrm{u}}  c_{\alpha})},\,0,\,0]^{\intercal}$ is the regressor vector, $\bm{\Theta}=\hat{T}$ is the parameters vector; $\bm{\delta}=[\delta_r, \,\delta_{\alpha}, \,\delta_{\theta}]^{\intercal}$ is the vector of lumped system disturbances and modeling errors across each channel, where $\bm{\hat{\delta}}=[\hat{\delta}_r,\,\hat{\delta}_{\alpha},\,\hat{\delta}_\theta]^{\intercal}$ is its estimate; and $\bm{H} \in \mathbb{R}^3$ denotes the nonlinear and gravitational terms vector defined as:
\begin{align*}
    \bm{H} & =  
    \begin{bmatrix}
        r \dot{\alpha}^2 - \ddot{x}_{\mathrm{b}} c_{\alpha} - \ddot{z}_{\mathrm{b}} s_{\alpha} -\textsl{g}s_{\alpha}\\
        (-2\dot{r}\dot{\alpha} + \dot{V} s_{\alpha} - \ddot{z}_{\mathrm{b}} c_{\alpha} -\textsl{g}c_{\alpha}) / r\\
        0\\
    \end{bmatrix}.
\end{align*}
\begin{assumption}\label{assump_bounded_dist}
  The modeling errors and external disturbances and their derivatives are bounded. 
\end{assumption}
\begin{assumption} \label{assump_slow_dist}
  The lumped error vector $\bm{\delta}$ is constant or slowly varying during a finite time interval, that is: $\lim_{t_1<t<t_2} \dot{\delta}_{\alpha}, \dot{\delta}_{r},\dot{\delta}_{\theta} \approx 0$. 
\end{assumption}
\par
Let $\theta'_{\mathrm{u,c}}$ be the desired UAV pitch angle to be generated by the outer-loop controller along with the total thrust command, $u_1$, which are calculated as:
\begin{equation} \label{eq_outer_loop_u1_thetac}
       u_{1} = \sqrt{u_{\alpha}^2+u_r^2},\qquad
        \theta'_{\mathrm{u,c}} = \frac{\pi}{2}-\alpha - \text{arctan}(u_{\alpha},u_r).
\end{equation}
Let $\theta_{\mathrm{u,c}}=\theta_{\mathrm{u,m}} \, \text{tanh}\big(\theta'_{\mathrm{u,c}}/\theta_{\mathrm{u,m}}\big)$
be a smooth and bounded version of $\theta'_{\mathrm{u,c}}$, with $ \theta_{\mathrm{u,m}} \in  (0,\frac{\pi}{2})$ being the absolute upper limit of the UAV's attitude angle.
The reference state vector to be followed is defined as $\bar{\bm{X}}_1=[\bar{r},\,\bar{\alpha},\,\theta_{\mathrm{u,c}}]^{\intercal}$.
Let the state error vector be defined as:
\begin{equation} \label{eq_e1}
            \bm{e}_1 = \bm{X}_1 - \bar{\bm{X}}_1.
\end{equation}
\par
The proposed control law including the radial and tangential thrust components for the outer-loop UAV's relative position controller, and the UAV's pitching torque, is defined as \cite{Kourani2021_PID_Like_Backstepping}: 
\begin{equation} \label{eq_U_p_PID}
    \begin{split}
        \bm{U} & = \bm{b}^{-1}\big[  -\bm{k}_P\bm{e}_1 -\bm{k}_D \dot{\bm{e}}_1-\bm{k}_I\bm{e}_1^{I} + \ddot{\bar{\bm{X}}}_1 -\bm{H} -\bm{\Phi}  \bm{\Theta} \big],\\
        \dot{\bm{e}}_1^I & = \bm{e}_1 + \bm{k}_1^{-1} \dot{\bm{e}}_1,
        \end{split}
\end{equation}
\noindent where $\bm{k}_{P}$, $\bm{k}_{D}$, $\bm{k}_{I}$, and $\bm{k}_{1} \in \mathbb{R}_{>0}^{3\times3}$ are controller gains that are defined next.
\begin{theorem} \label{theorem_ss_controller}
    Consider the UAV$-$buoy's relative position dynamics in (\ref{eq_UAV_polar_frame}), and the state space representation of the system in (\ref{eq_ss_rat_model}). Suppose that Assumptions \ref{assump_bounded_dist} and \ref{assump_slow_dist} hold true; the control law in (\ref{eq_outer_loop_u1_thetac}) and (\ref{eq_U_p_PID}) generates the total thrust, $u_1$, and the UAV's desired pitch angle, $\theta_{\mathrm{u,c}}$, that can stabilize the system, and reduce the tracking error to zero in finite time for a set of gains $\bm{k}_1$, $\bm{k}_{2}$, and $\bm{\gamma} \in \mathbb{R}_{>0}^{3\times3}$,
    such that $\bm{k}_{P} = \bm{I}_3+\bm{k}_1 \bm{k}_2$, $\bm{k}_{D} = \bm{k}_1 + \bm{k}_2$, $\bm{k}_{I} = \bm{\gamma} \bm{k}_1$, with $\bm{I}_3$ being the identity matrix. If Assumption~\ref{assump_slow_dist} does not hold, the tracking error reduces to a small region neighboring the origin in finite time.
\end{theorem}
\begin{proof}
    The backstepping control design, involving two steps, is employed, and the Lyapunov function $\mathcal{V}_1=\frac{1}{2} \bm{e}_1^{\intercal}\bm{e}_1$ is proposed. Its derivative is expressed as: 
    $\dot{\mathcal{V}}_1 = \bm{e}_1^{\intercal}\dot{\bm{e}}_1$.
    Since $\dot{\bm{e}}_1$ does not explicitly include a control input, we continue the control design process for a second step. The virtual control input to stabilize $\bm{e}_1$ is defined as:
    $\bm{\Upsilon} = \dot{\bar{\bm{X}}} - \bm{k}_1 \bm{e}_1$.
    Next, we define the virtual rates error as:  $\bm{e}_2 = \dot{\bm{X}}_1 - \bm{\Upsilon}$.
    \par
    By defining a second Lyapunov function:
    \begin{align*}
        \mathcal{V}_2 = \frac{1}{2} \bm{e}_1^{\intercal}\bm{e}_1 + \frac{1}{2}\bm{e}_2^{\intercal}\bm{e}_2 + \frac{1}{2}\tilde{\bm{\delta}}^{\intercal}\bm{\gamma}^{-1} \tilde{\bm{\delta}},
    \end{align*}
    \noindent where $\tilde{\bm{\delta}} = \hat{\bm{\delta}} - \bm{\delta}$,
    then by differentiating and combining it with $\mathcal{V}_1$, we get:
    \begin{align*}
        \dot{\mathcal{V}}_2 & = \bm{e}_1^{\intercal} \dot{\bm{e}}_1 + \bm{e}_2^{\intercal} \dot{\bm{e}}_2 + \tilde{\bm{\delta}}^{\intercal} \bm{\gamma}^{-1} \dot{\hat{\bm{\delta}}} \\ 
        & = \bm{e}_1^{\intercal}(\bm{e}_2 - \bm{k}_1 \bm{e}_1) + \bm{e}_2^{\intercal} (\bm{H}  + \bm{b}\bm{U} +\bm{\Phi}\bm{\Theta} + \bm{\delta} - \dot{\bm{\Upsilon}} ) \\
        & + \tilde{\bm{\delta}}^{\intercal} \bm{\gamma}^{-1} \dot{\hat{\bm{\delta}}}.
    \end{align*}    
    Next, we choose the control inputs and the lumped modeling and disturbances errors' update rates such that $\dot{\mathcal{V}}_2$ becomes negative semi-definite:
    \begin{equation} \label{eq_U_rat_lyapunov}
        \begin{split}
            \bm{U}  & = \bm{b}^{-1} \big( -\bm{H} -\bm{\Phi}\bm{\Theta} -\hat{\bm{\delta}} + \dot{\bm{\Upsilon}} - \bm{e}_1 - \bm{k}_2 \bm{e}_2 \big),  \quad  
            \dot{\hat{\bm{\delta}}} = \bm{\gamma} \bm{e}_2,
        \end{split}
    \end{equation}
    and we get $\dot{\mathcal{V}}_2 = -\bm{e}_1^{\intercal} \bm{k}_1 \bm{e}_1 - \bm{e}_2^{\intercal} \bm{k}_2 \bm{e}_2$. 
    Thus, the asymptotic convergence of $\mathcal{V}_2$ to zero can be obtained via Barbalat's lemma under Assumption~\ref{assump_slow_dist}. If strong wind and wave disturbances exist, meaning the violation of Assumption~\ref{assump_slow_dist}, the control law will still achieve stability and finite tracking error, which can be reduced by increasing the controller gains up to a level that overcomes the disturbances mismatch effect on $\dot{\mathcal{V}}_2$. Finally, by substituting $\dot{\bm{\Upsilon}}$ and $\bm{e}_2$ in (\ref{eq_U_rat_lyapunov}), and setting $\bm{e}_1^{I} := \hat{\bm{\delta}}  (\bm{\gamma} \bm{k}_1)^{-1}$, the PID-like control law in (\ref{eq_U_p_PID}) is obtained.
\end{proof}
\par
\subsubsection{Buoy Surge Velocity Control Law} \label{sec_surge_vel_control_law}
Consider the UAV dynamics in $\mathcal{W}'$, while following the polar coordinates notation as presented in (\ref{eq_UAV_polar_frame}), and let the buoy's velocity error be defined as $e_V = V-\bar{V}$.
The buoy velocity model can be expressed in the generic case, i.e. variable radial position, as:
\begin{equation} \label{eq_velocity_elevation_model}
    \dot{V}  =   H_V  - T/ (m_{\mathrm{u}}  c_{\alpha}) + u_r^{\mathrm{v}} / (m_{\mathrm{u}}  c_{\alpha}), \\
\end{equation}
\noindent where
    $H_V = ( r \dot{\alpha}^2 - \ddot{r} - \ddot{z}_{\mathrm{b}} s_{\alpha} -\textsl{g}\,s_{\alpha}) / c_{\alpha}$. 
A control law can be designed for the surge velocity in a similar fashion as described in Section~\ref{sec_rel_pos_control_law}, with a difference that only one step is required in the backstepping process. The resulting control law is given by:
\begin{equation} \label{eq_u_r_V}
    u_r^{\mathrm{v}} = \hat{T}_\mathrm{c} + m_{\mathrm{u}} c_{\alpha} \big( -H_V  + \dot{\bar{V}} - k_{PV} e_{V} - k_{IV} e_{V}^I \big),\;\;\; \dot{e}_{V}^I = e_V,
\end{equation}
\noindent where $k_{PV}$ and $k_{IV} \in \mathbb{R}_{>0}$ are controller gains. The results of Theorem~\ref{theorem_ss_controller} relative to stability and tracking apply.
\begin{remark} \label{rem_tension_estimation}
    Cable tension can be either directly measured (e.g. load cell) to improve the tracking performance and the system's overall safety, or it can be estimated via an observer design based on cable  disturbance  estimation  methods. 
    However, this internal force, $T$, and its estimate, $\hat{T}$,  should not be confused with the term $\hat{T}_{\mathrm{c}}$ used in the control law (\ref{eq_u_r_V}), and representing the required tensile force to manipulate the buoy.
    One simple realization is obtained based on (\ref{eq_cable_tension_equilibrium}), such that: 
    \begin{equation} \label{eq_tension_feedforward}  
        \hat{T}_{\mathrm{c}} = D_{\mathrm{b},11,0} \bar{V} / c_{\alpha}, 
    \end{equation}
    \noindent where $D_{\mathrm{b},11,0} = \bar{C}_{\mathrm{S},1} \mathrm{A}_{\mathrm{wt},0} \frac{1}{2} \rho_{\mathrm{w}} |\bar{V}|$, with $\mathrm{A}_{\mathrm{wt},0}$ being the zero-tension whetted area, and $\bar{C}_{\mathrm{S},1}$ being the surge skin friction coefficient at $\bar{V}$,
    which yields a fair, yet not very accurate, estimate. However, the proposed controller can compensate for the estimation error as will be proven next.
    For a sample cable tension estimation based on disturbance observation, readers are referred to \cite{Tal2018AccurateTrackingTR}.
\end{remark}
\begin{remark} \label{rem_robust_performance}
    Practically, robust performance of the proposed control laws is guaranteed by choosing large-enough $\bm{k}_2$ gains for a wide operating range, even if Assumption~\ref{assump_slow_dist} is violated \cite{Hu2010DIARC_Gantry}. 
\end{remark}
%
\section{Simulations} \label{sec_simulation}
%
In this section, we provide simulation results that demonstrate the fidelity of the tethered UAV$-$buoy system model and the performance of the designed controller. We first define the settings and parameters used for the devised simulation scenarios, which include various operating conditions to validate the proposed system.
To challenge the control law's performance towards real-life implementation, the tethered UAV$-$buoy system model is incorporated in the simulator developed in this work, while including deviation from the described model used by the control law, including the UAV's propellers motor dynamics, wind gusts, and non-exact state feedback.
\subsection{Simulation Settings}
To validate the proposed UAV$-$buoy system with the designed SVCS, a series of simulations is performed in the MATLAB Simulink\,\textsuperscript{\tiny\textregistered} environment.
We consider a quadrotor UAV and a simplified homogeneous cuboid buoy with the dimensions and parameters listed in Table~\ref{tab_system_param}. The quadrotor UAV motor dynamics are modeled as a first-order low-pass filter with a time constant $\tau_m=\SI[unitsep=medium]{0.05}{\second}$, and its total thrust and pitch torque are bounded such that $\|u_1\| \leq \SI[unitsep=medium]{160}{\newton}$, and $\|u_2\| \leq \SI[unitsep=medium]{11.2}{\newton \meter}$. The mass of the buoy is chosen such that the buoy is one quarter immersed under no external loads based on the balance between the gravitational and buoyancy forces, that is $m_{\mathrm{b}} := \rho_{\mathrm{w}} \curlyvee_{\mathrm{b}}/4$.
%
\begin{table}
    \caption{Tethered UAV$-$buoy model parameters}
    \begin{center}
    \begin{tabular}{ p{1.0cm} p{1.2cm} p{1.0cm} |p{1.0cm} p{1.1cm} p{1.0cm} } 
        \hline 
        Par. & Value & Unit & Par. & Value & Unit \\
        \Xhline{2\arrayrulewidth}
        \hline
        $l_{\mathrm{b}}$  &  0.8  &  $\SI[unitsep=medium]{}{\meter}$   & $m_{\mathrm{u}}$  &  1.8  &   $\SI[unitsep=medium]{}{\kilogram}$  \\
        $h_{\mathrm{b}}$  &  0.25 &  $\SI[unitsep=medium]{}{\meter}$    & $J_{\mathrm{u}}$  &  0.03  &  $\SI[unitsep=medium]{}{\kilogram \square\meter}$ \\
        $m_{\mathrm{b}}$  & 12.5   & $\SI[unitsep=medium]{}{\kilogram}$    & $\theta_{\mathrm{u,m}}$  &  $\pi/4$  &  $\SI[unitsep=medium]{}{\radian}$ \\
        $a_{11}$  &  $0.625$  & $\SI[unitsep=medium]{}{\kilogram}$    & $l$  &  7 & $\SI[unitsep=medium]{}{\meter}$\\
        $a_{33}$  &  $12.5$  & $\SI[unitsep=medium]{}{\kilogram}$ &  $\epsilon_{T}$  &  5  & $\SI[unitsep=medium]{}{\newton}$\\
        $b_{11}$  &  0  & $\SI[unitsep=medium]{}{\newton \second /\meter}$ & $\epsilon_{\curlyvee}$  & 0.05  & - \\      
        $b_{33}$  & 27.5  & $\SI[unitsep=medium]{}{\newton \second / \meter}$ & $\epsilon_{m}$  &  0.1  & - \\
        $\nu_{\mathrm{w}}$ & 1.78e-6 & $\SI[unitsep=medium]{}{\square\meter/\second}$ & $\rho_{\mathrm{w}}$  &  1000  & $\SI[unitsep=medium]{}{\kilogram /\meter\cubed}$\\ 
        $\textsl{g}$  & 9.81 & $\SI[unitsep=medium]{}{\meter/\square\second}$ & $\rho_{\mathrm{a}}$  &  1.22  & $\SI[unitsep=medium]{}{\kilogram /\meter\cubed}$\\ 
        \hline
            \end{tabular}
    \label{tab_system_param}
    \end{center}
\end{table}
%
The buoy's immersed volume is then defined as:
\begin{equation} \label{eq_buoy_wetted_area} 
   \curlyvee_{\mathrm{im}}(\Delta h) = \begin{cases}
    \curlyvee_\mathrm{b} & \text{if $\Delta h > \frac{h_{\mathrm{b}}}{2}$},\\
    0 & \text{if $\Delta h < -\frac{h_{\mathrm{b}}}{2}$},\\
    \curlyvee_\mathrm{b}/2 + l_{\mathrm{b}} h_{\mathrm{b}} \Delta h & \text{otherwise},\\
  \end{cases}
\end{equation}
\noindent where $\Delta h = \zeta(x_{\mathrm{b}},t) - z_{\mathrm{b}}(t)$. 
The wetted area is calculated as:
\begin{equation} \label{eq_buoy_wetted_area_Aw} 
   \mathrm{A}_{\mathrm{wt}}(\Delta h) = \begin{cases}
    4 l_{\mathrm{b}} h_{\mathrm{b}} & \text{if $\Delta h > \frac{h_{\mathrm{b}}}{2}$},\\
    0 & \text{if $\Delta h < -\frac{h_{\mathrm{b}}}{2}$},\\
    l_{\mathrm{b}} h_{\mathrm{b}} + 2 l_{\mathrm{b}} (\frac{h_{\mathrm{b}}}{2} + \Delta h) & \text{otherwise}.\\
  \end{cases}
\end{equation}
\noindent The resulting added mass and damping are calculated as described in Section \ref{subsec_buoy_model}, and their values are presented in Table~\ref{tab_system_param}.
The buoy's skin friction coefficients in its body $x$- and $z$-directions can be estimated as $C_{S,i} = 0.075/(\log_{10}\text{Re}-2)^2$,
where $\text{Re}= \frac{|V_{\mathrm{r}}| l_{\mathrm{b}}}{\nu_{\mathrm{w}}} \in \mathbb{R}_{\geq 0}$ is the Reynolds number, limited to turbulent flows ($\text{Re} > 10^5$), with $\nu_{\mathrm{w}}$ being the water's  kinematic viscosity \cite{Lewis1988}.
To detect the coupling state of the system (coupled / decoupled), we rely on the tension estimation in the second case of (\ref{eq_cable_tension_estimation}).
\par
\sloppy To properly evaluate the performance of the proposed SVCS design, a Cartesian-based nominal controller (CBNC) that uses a PID control law in its outer-loop, and without supervision of a state-machine, is implemented for benchmarking purposes. It consists of a velocity ($\dot{x}$) controller and an elevation ($z$) controller, with gains $\bm{k}_{P,CBNC}=diag(7,\,3)$,  $\bm{k}_{I,CBNC}=diag(1.2,\,1)$, and  $\bm{k}_{D,CBNC}=diag(5,\,2)$, respectively. 
The SVCS gains are selected as $\bm{k}_1=diag(16.9,\,4.6,\,7.5)$,  $\bm{k}_2=diag(2.6,\,2.4,\,2.5)$,  $\bm{\gamma}=diag(0.5,\,0.3,\,0.3)$, 
$k_{PV}=25$, and $k_{IV}=12$.
\par
The feedback signals are assumed to be available from sensor measurements and estimations, and are modeled as follows: the UAV's pose is virtually obtained from an on-board Global Positioning System / Inertial Navigation System (GPS/INS) module, and the elevation angle and the radial distance are virtually obtained from a stereo camera system.
In simulation, this is mimicked by augmenting the feedback states by a filtered Gaussian noise with the corresponding state-of-the-art accuracy of each sensor before being used by the controller.
With $\mathrm{mav}(\tilde{\bigcdot})$ denoting the mean absolute value of the estimation error of entity $(\bigcdot)$, we set $\mathrm{mav}(\tilde{x}_{\mathrm{u}}) = \SI[unitsep=medium]{0.02}{\meter} = \mathrm{mav}(\tilde{z}_{\mathrm{u}}) = \SI[unitsep=medium]{0.02}{\meter} $, $\mathrm{mav}(\tilde{\theta}_{\mathrm{u}}) = \SI[unitsep=medium]{0.5}{\degree}$, $\mathrm{mav}(\tilde{\alpha}) = \SI[unitsep=medium]{0.16}{\degree}$, and $\mathrm{mav}(\tilde{r}) = \SI[unitsep=medium]{0.02}{\meter}$.
Subsequently, the buoy's states are determined from (\ref{eq_UAV_polar_coordinates}), (\ref{eq_dot_r_in_W_prime}), and (\ref{eq_ddot_r_in_W_prime}). More details on state estimation is given in Section~\ref{sec_states_estimation}.
\par
\subsection{Velocity Bounds} \label{subsec_velocity_bounds}
The constraints' bounds ($\epsilon_{T}$, $\epsilon_{\curlyvee}$, and $\epsilon_{m}$) presented in Table~\ref{tab_system_param} are mainly based on the expected buoy$-$water relative velocity, in addition to the buoy's shape, weight, and skin friction.
A possible command velocity range of $S_{\bar{V}}=(-19.0,-3.1) \cup (2.1,18.0) \SI[unitsep=medium]{}{\meter \per \second}$ is calculated from (\ref{eq_valid_vel_set}) under no-wave condition (Assumption~\ref{assump_equilibrium_NO_wave}).
\par
In the presence of waves, the feasible working velocity with no violation of constraint (\ref{eq_water_contact_condition}) reduces from above, and can be quantified by referring to (\ref{eq_fly_over_constraint}) and (\ref{eq_u1_bar_Vol_bar_b}) as follows. We solve for $\bar{\curlyvee}_{\mathrm{im}}$ to get $\bar{h}_{\mathrm{im}}$, then find $\Delta h_{\mathrm{amp}}$ for some $(\bar{V},\bar{\alpha})$, under different wave conditions.
Fig.~\ref{fig_dynamic_amplification} provides the buoy's heave dynamic amplification results under excitation of a single fully-developed wave component \cite{Fossen2011}, with $\bar{\alpha}=\SI[unitsep=medium]{45}{\degree}$ and $U_{\mathrm{c}}=0$. To have a unified representation of $\bar{h}_{\mathrm{im}}$, the Stokes drift effect is neglected in calculating $\bar{V}_r$.  
The natural frequency of the buoy, calculated from (\ref{eq_buoy_nat_freq_and_damp}), is $\omega_\mathrm{b}=\SI[unitsep=medium]{8.9}{\radian\per\second}$.
\par
Fig.~\ref{fig_dynamic_amplification} can be interpreted as follows:
for a given sea condition with wave amplitude and period $\{A_n,\mathrm{T}_n\}$, the buoy hops over the waves (`fly-over' condition) when its horizontal (forward or backward) velocity, $V$, falls outside the shaded area (dome) formed by the $\bar{h}_\mathrm{im}$ curve, for a given $\Delta h_{\mathrm{amp},n}$ (colored plots corresponding to various wave amplitudes and periods).
Sample zones, where the buoy `fly-over' condition does not occur, are marked on top of the figure as $S_{\bar{V},n}^{\mathrm{fo}}$.
%
%
The comprehensive results captured by  Fig.~\ref{fig_dynamic_amplification} show that the system operation is direction-dependent, and they also serve as a reference for predicting the performance of the buoy in terms of heave oscillation and `fly-over' phenomenon under different wave conditions, ranging from high-frequency low-amplitude waves to low-frequency high-amplitude ones, and even for superposition of various waves, as will be demonstrated in the subsequent sections.
We note that the above analysis is provided for a buoy of known characteristics (Table~\ref{tab_system_param}), and serves as a guideline for the system performance.
\par
%
\begin{figure}
\centerline{\includegraphics[width=3.45in]{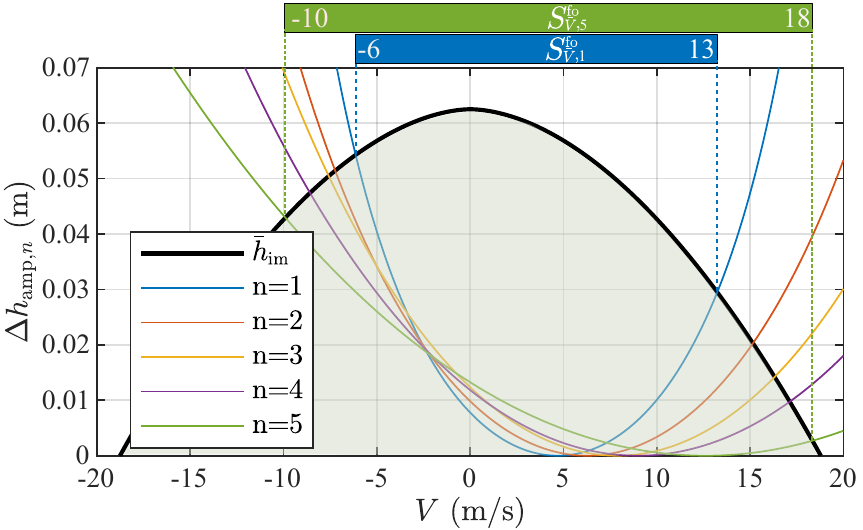}}
\caption{Buoy's heave dynamic amplification, $\Delta h_{\mathrm{amp},n}$, under excitation of different fully-developed single wave components $\{A_n,\mathrm{T}_n\}$ of the sets $A = \{0.27,0.61,,1.2,1.5,3.3\}/2 \, \SI[unitsep=medium]{}{\meter}$, and $\mathrm{T}=\{3,4,5,5.7,8\} \, \SI[unitsep=medium]{}{\second}$. The mean buoy's immersed height, $\bar{h}_{\mathrm{im}}$, draws the boundary dome for the `fly-over'-free region.}
\label{fig_dynamic_amplification}
\end{figure}
%
%
\subsection{Simulation Scenarios}
We validate the fidelity of the derived system model and evaluate the performance of the designed controller in four cases: C1, C2, C3, and C4. All cases include a wind gust of $u_{\mathrm{wd}}= \SI[unitsep=medium]{-3}{\meter\per\second}$ and a water current component $U_{\mathrm{l}}= \SI[unitsep=medium]{-0.5}{\meter\per\second}$. The scenarios are described as follows: 
\begin{itemize}
    \item C1: wind gust and water current only.
    \item C2: wind gust, water current, and moderate waves with two wave components ($N=2$), such that: 
    $A_1=\SI[unitsep=medium]{0.135}{\meter}$, $d_1=1$, $\mathrm{T}_1=3\,\SI[unitsep=medium]{}{\second}$, and $\sigma_1=\pi$;    
    $A_2=\SI[unitsep=medium]{0.75}{\meter}$, $d_2=1$, $\mathrm{T}_2=5.7\,\SI[unitsep=medium]{}{\second}$, and $\sigma_2=0$.
    \item C3: high-frequency small-amplitude waves (head-seas), with     $A_1=\SI[unitsep=medium]{0.135}{\meter}$, $d_1=-1$, $\mathrm{T}_1=3\,\SI[unitsep=medium]{}{\second}$, and $\sigma_1=0$. 
    \item C4: high-amplitude low-frequency waves (head-seas), with     $A_1=\SI[unitsep=medium]{1.65}{\meter}$, $d_1=-1$, $\mathrm{T}_1=7\,\SI[unitsep=medium]{}{\second}$, and $\sigma_1=0$. 
\end{itemize}
\noindent Note that the wave components definitions in each scenario is independent from the other scenarios. Sample visual illustrations of the environments in C1 and C2 are given in Fig.~\ref{fig_simulation_scenarios}, which are generated via the custom-built simulator that we specifically developed to serve as a physics engine and provide live animations for tethered UAV$-$buoy locomotives.
In both cases, C1 and C2, the buoy is commanded to accelerate to reach an inertial velocity $\bar{V}=\SI[unitsep=medium]{5}{\meter\per\second}$, after which it gradually decelerates to $\SI[unitsep=medium]{0}{\meter\per\second}$ then to $\SI[unitsep=medium]{-4}{\meter\per\second}$. The desired reference mean sea level altitude is $\bar{z}_{\mathrm{u}}=\SI[unitsep=medium]{5.0}{\meter}$, which corresponds to a mean elevation angle of $\bar{\alpha}_0 = 45^{\circ}$.
The system is initiated in the decoupled state, and its velocity is initiated to be equivalent to the zero-time water velocity via (\ref{eq_wave_vel}) and (\ref{eq_current}). 
Based on Assumption \ref{assump_stable_buoy_rot_dyn}, the buoy's pitch angle is calculated by differentiating (\ref{eq_water_surface_elevation}) with respect to $x_{\mathrm{b}}$:
\begin{equation} \label{eq_buoy_tilt_angle}
        \theta_{\mathrm{b}} = 
        \mathrm{atan}\Big( \sum_{n}^{N} A_n k_n \cos(d_n \omega_n t - k_n x_{\mathrm{b}} + \sigma_n)
        \Big).    
\end{equation}
\par
While cases C1 and C2 provide a baseline evaluation of the proposed robotic system and its controller, cases C3 and C4 challenge its performance in extreme cases, i.e. under fast oscillations (C3), and high amplitude undulations (C4). A low-slope ramp velocity input ($V = 0.25 \,t$) is applied to carefully capture the performance of the system at different velocities, and head-seas are considered to emphasize and validate the universality of the buoy's dynamic heave performance captured in Fig.~\ref{fig_dynamic_amplification}. 
%
\begin{figure}
\centerline{\includegraphics[width=3.45in]{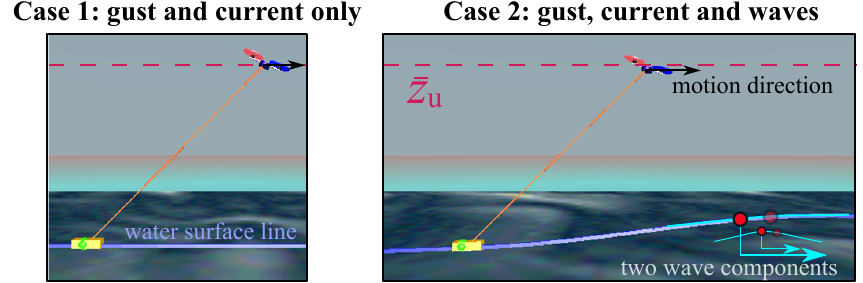}}
\caption{Sample screenshots from animations of two  simulation scenarios (C1 and C2) in true scale. Animations are generated via a custom-built simulator that is specifically developed to serve as a physics engine for tethered UAV$-$buoy locomotives.}
\label{fig_simulation_scenarios}
\end{figure}
%
\subsection{Simulation Results and Discussion}
The simulation results for C1 and C2 are shown in Fig.~\ref{fig_C1_C2}a and Fig.~\ref{fig_C1_C2}b, respectively. 
%
\begin{figure}
\centerline{\includegraphics[width=3.45in]{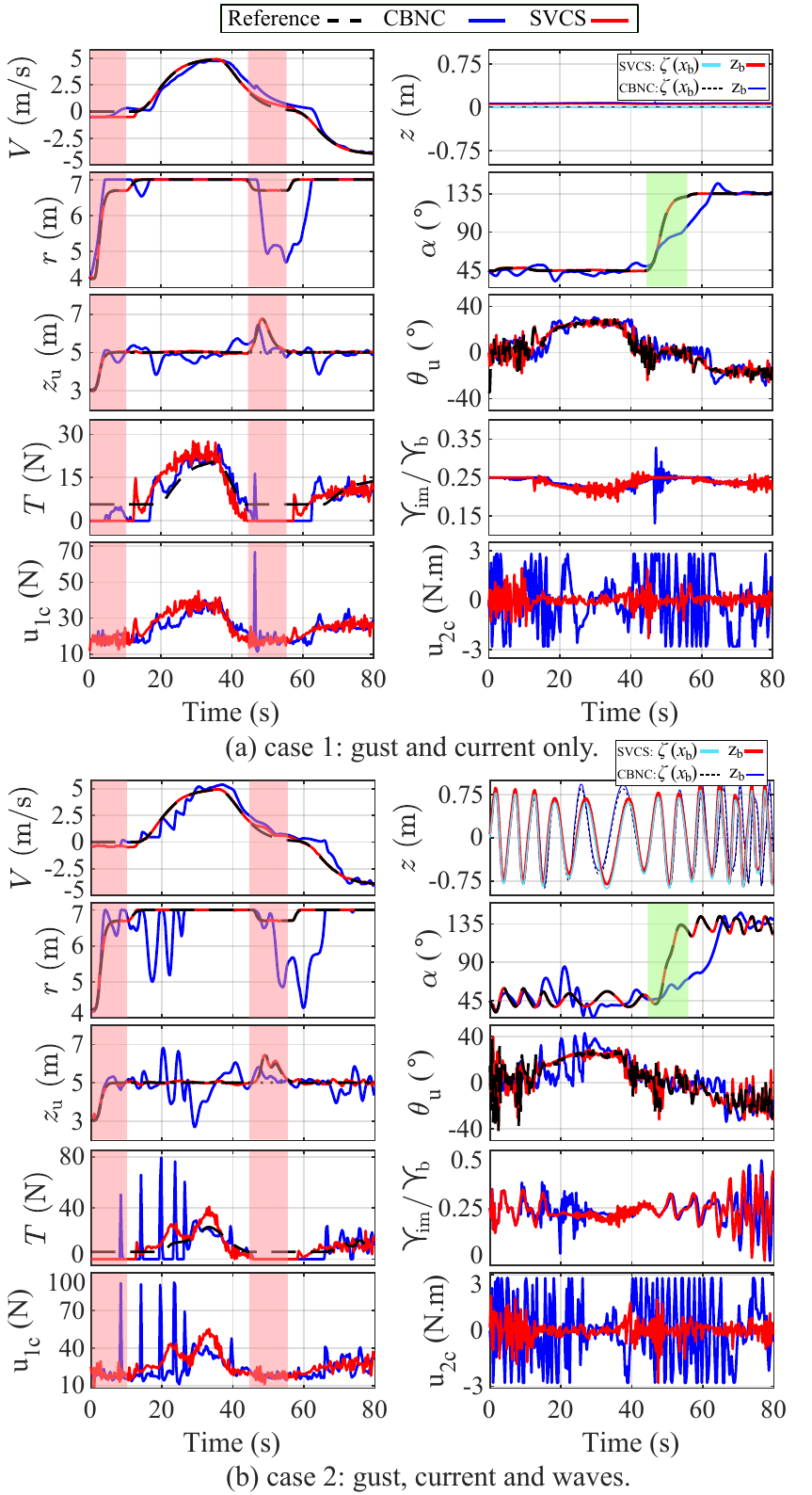}}
\caption{States and control inputs for the simulation scenarios C1 and C2, with both the state machine-supervised surge velocity control system (SVCS) and a standard Cartesian-based nominal UAV controller (CBNC). The region in red marks when the mode is not `pulling', and the region in green marks when the mode is not `repositioning'.}
\label{fig_C1_C2}
\end{figure}
%
In both cases, the quadrotor UAV equipped with the SVCS is able to pull the buoy at the desired velocities ($V$) without overshoot,  with minimal fluctuations in velocity ($V$) and elevation ($z_{\mathrm{u}}$), while not violating constraint (\ref{eq_water_contact_condition}) as indicated by the immersed volume plot ($\curlyvee_{\mathrm{im}}/\curlyvee_{\mathrm{b}}$), and without unnecessarily decoupling the system (as seen in the $r$ subplot). 
The resulting commands to the UAV, $u_{1\mathrm{c}}$ and $u_{2\mathrm{c}}$, are bounded and free of high-frequency chattering.
On the other hand, the Cartesian-based controller without state machine supervision (CBNC) results in significantly larger velocity ($V$) and elevation ($z_{\mathrm{u}}$) fluctuations, reaching up to $\SI[unitsep=medium]{2}{\meter\per\second}$ and $\SI[unitsep=medium]{2.2}{\meter}$, respectively, in the wavy environment (C2). 
\par
The proposed controller adjusts the elevation angle, $\alpha$, while the buoy elevation, $z_{\mathrm{b}}$, changes $-$ driven by the contour-following behaviour of the buoy under long waves excitation $-$ to prevent unnecessary UAV vertical motion ($z_{\mathrm{u}}$) as evident in Fig.~\ref{fig_C1_C2}b. The adjustments are such that $\bar{\alpha}$ varies in response to changes in the buoy's elevation, $z_{\mathrm{b}}$, which is proportional to the wave encounter frequency.
It is also observed that the elevation angle ($\alpha$) and pitch angle ($\theta_{\mathrm{u}}$) are smooth and stable, and exhibit small tracking error for the SVCS.
We note that the reference UAV pitch angle, $\theta_{\mathrm{u,c}}$, for the CBNC is not plotted for figure clarity purposes, since both systems possess the same inner-loop controller, in addition to the fact that CBNC has no reference elevation angle, $\bar{\alpha}$, and radial position, $\bar{r}$.
\par
The SVCS-controlled UAV achieves the desired surge velocity of the buoy by adjusting the cable tension, $T$, in an appropriate and relatively smooth manner as seen in ($T$).
Contrarily, the CBNC has no direct control of the cable tension and the radial position of the UAV, which leads to repeated large input pulses that deteriorate the transient performance. 
Finally, it is observed that the change in the immersed volume of the buoy greatly depends on the encounter frequency.
It is also noticed that the buoy remains in contact with the water surface (($\curlyvee_{\mathrm{im}}/\curlyvee_{\mathrm{b}}$)) for both controllers.
\par
To quantify the performance of the two controllers, the trajectory tracking errors of $V$ and $z_{\mathrm{u}}$, and the energy consumed by the UAV (calculated per \cite{[Christoph_CDC]}) are reported in Table~\ref{tab_tracking_error_energy}. The SVCS results in an average reduction in the tracking error of $88\,\%$ and in energy consumption of $42\,\%$ versus the CBNC.
\par
%
\begin{table}
    \caption{Tracking Error and Consumed Energy}
    \begin{center}
    \begin{tabular}{ P{0.7cm} | P{0.85cm} P{0.85cm} | P{0.85cm} P{0.85cm} | P{0.85cm} P{0.85cm}} 
        \hline 
        \multirow{3}{*}{Case} & \multicolumn{2}{c|}{$V$, error}  & \multicolumn{2}{c|}{$z_{\mathrm{u}}$, error} & \multicolumn{2}{c}{cons. en.}  \\ 
        & \multicolumn{2}{c|}{($\SI[unitsep=medium]{}{cm\per\second}$)}  & \multicolumn{2}{c|}{($\SI[unitsep=medium]{}{cm}$)} & \multicolumn{2}{c}{ ($\SI[unitsep=medium]{}{\kilo\joule}$)}  \\ \cline{2-7} 
        &  CBNC   &  SVCS &  CBNC  & SVCS  &  CBNC  & SVCS \\
        \Xhline{2\arrayrulewidth}
        C1  &  36.4  &  5.4 &  28.4 & 2.7 & 111.8 &  58.9 \\
        C2  &  61.3  &  6.1 &  42.8 & 5.9 & 93.4 &  61.2 \\
        \hline
    \end{tabular}
    \label{tab_tracking_error_energy}
    \end{center}
\end{table}
\begin{remark} \label{rem_correlated_effect}
    Expressing the SVCS in polar coordinates yields a correlated control performance, which means that each control channel, ($u_{r}$ and $u_{\alpha}$), independently affects one control parameter ($\alpha$ or $V$). However, this is not the case for the Cartesian-based controller (CBNC), where each of the $\dot{x}$- and $z$-control channels has a dual effect on each control parameter, which results in a degraded performance. 
\end{remark}
\par
In summary, the CBNC does not cope with the introduction of waves to prevent them from disturbing the system in an unpredictable manner, nor it respects the system configuration, whereas the SVCS shows its disturbance-rejection property in significantly attenuating the waves' effect even without knowing their characteristics. 
All of the above factors, combined, justify the design of the relatively complex SVCS for the proposed marine locomotive UAV system.
\par
The performance of the SVCS-equipped locomotive system against high-frequency and high-amplitude waves are shown in Fig.~\ref{fig_C3_C4}a and Fig.~\ref{fig_C3_C4}b, respectively. The first separation of the buoy from the water surface occurs at $t = \SI[unitsep=medium]{21}{\second}$ and $V=\SI[unitsep=medium]{5}{\meter\per\second}$ in C3, and $t = \SI[unitsep=medium]{41}{\second}$ and $V=\SI[unitsep=medium]{11}{\meter\per\second}$ in C4, which are marked by the yellow strips in their respective subplots.
The tracking accuracy in $V$ and $z_{\mathrm{u}}$ demonstrates that the proposed SVCS performs well in the considered extreme scenarios, as long as they are within the working zones established in Fig.~\ref{fig_dynamic_amplification}. Beyond these zones, i.e. after the instances marked by the yellow strips in Fig.~\ref{fig_C3_C4}, the buoy `fly-over' deteriorates the system performance, which manifests as jumps of the buoy above the waves as exhibited in the $z$ subplot.
\par
%
\begin{figure}
\centerline{\includegraphics[width=3.45in]{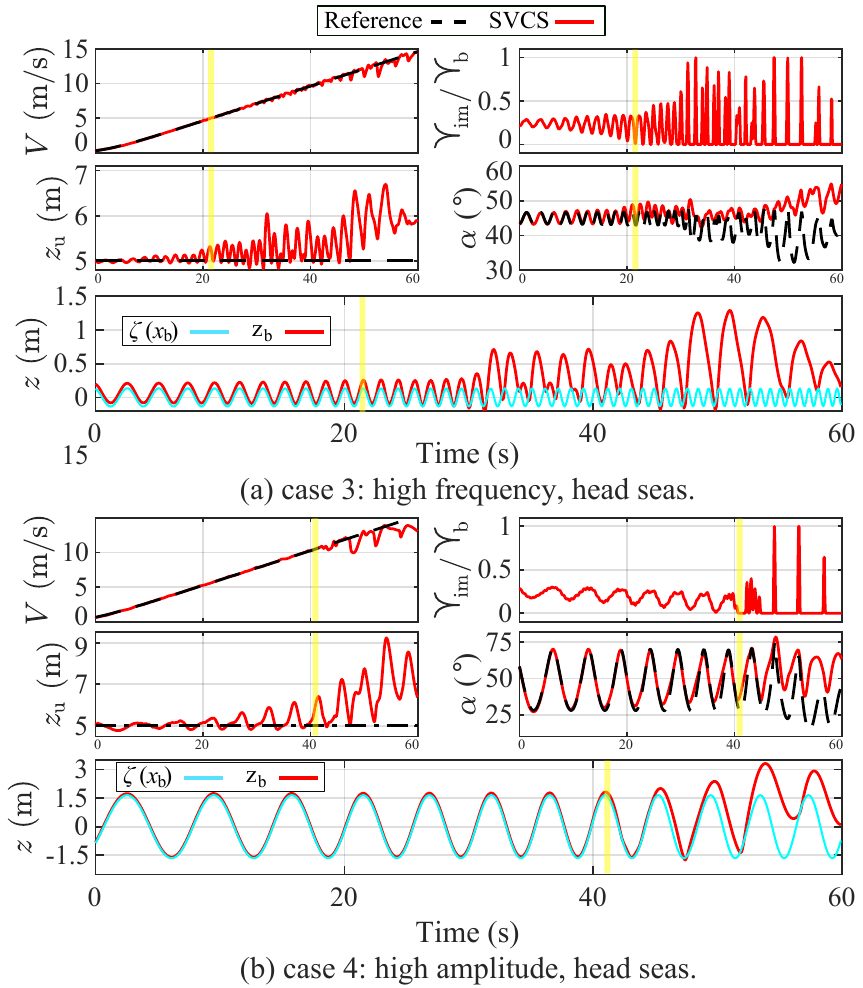}}
\caption{Simulation scenarios C3 (high-frequency small-amplitude waves) and C4 (high-amplitude low-frequency waves) to illustrate the SVCS performance against extreme sea conditions. The yellow strips mark the first buoy$-$water separation (`fly-over') in each case.}
\label{fig_C3_C4}
\end{figure}
%
\section{Practical Considerations} \label{sec_practical_considerations}
For the proposed system to lend itself well to physical implementation, we target in this section critical aspects that are essential to experimentally validate the proposed system, in preparation for its deployment in real-life.
\subsection{States Estimation} \label{sec_states_estimation}
%
%
The SVCS requires the following states for feedback: 
$r$, $\dot{r}$, $\ddot{r}$, $\alpha$, $\dot{\alpha}$, $\theta_{\mathrm{u}}$, $\dot{\theta}_{\mathrm{u}}$, $V$, $\ddot{x}_{\mathrm{b}}$, and $\ddot{z}_{\mathrm{b}}$.
Unlike most applications of the tethered UAV problem, the tether of the UAV$-$buoy system is not anchored to a fixed point  as in \cite{Sandino2014TetherLandingNoGPS,Nicotra2014,Tognon2017}; furthermore, knowing that the system is allowed to decouple, resulting in a slack cable, observer-based methods that target system state estimation in the taut cable case, as in  \cite{Tognon2020AerialRobotsTethers}, cannot be solely employed for state feedback. 
However, since the UAV$-$buoy system works above the water surface, it is practical to assume that GPS coverage is available, which allows for UAV pose estimation in the inertial frame using a GPS sensor and an inertial measurement unit (IMU) that is equipped with a magnetometer \cite{Tognon2020AerialRobotsTethers}. We also note that if the system is designed to operate in the vicinity of a marine structure, Real-Time Kinematic (RTK) GPS solutions can also be utilized to attain more accurate inertial state estimation \cite{Daly2015Quadrotor_Husky_RTK}.
To solve the state estimation problem when the cable is slack, the UAV can be equipped with a stereo camera to detect and estimate the buoy location in the camera frame using special-purpose algorithms \cite{Hussein2016StereoLaserObstacles,Schauwecker2014DualStereoNavigation}, from which the UAV's relative radial coordinates to the buoy ($r$ and $\alpha$), and the buoy's velocity ($V$) can be estimated \cite{Pizzoli2014MonocularReconstruction}. We note that a monocular camera can provide adequate accuracy for the control problem on hand only if the buoy's dimensions are known a priori \cite{Falanga2017NarrowGaps}. As for laser-baser sensory equipment, they are susceptible to sun rays exposure and water surface refraction, which deems them unsuitable for such applications.
Last but not least, using encoders placed on the UAV can help with measuring the cable's length and elevation angle in the taut-cable case, and a force sensor (e.g. load cell) allows measurement of the cable's tension \cite{Tognon2020AerialRobotsTethers}, thus providing the control system with additional information to improve its performance.
\subsection{Power considerations}
To make the system more energy efficient, it can be designed to allow the UAV to land on the buoy, or to float directly on the water surface, during long standby periods \cite{Shao2019UAV-USV}. Another alternative to further extend the work-time of the system is to integrate an umbilical power cable within the tether.
Furthermore, using an umbilical power line with power banks stationed on the buoy can be more efficient than increasing the on-board power capacity of the UAV under specific conditions. Also, a small relative buoy$-$water velocity and a streamlined buoy shape can result in better energy efficiency.
We note that in case of large umbilical power transmission cables, their mass cannot be neglected and must be compensated for in the control law of the coupled dynamics model as in \cite{Nguyen2019} and \cite{Kourani2021_Tethered_UAV_Buoy}; and in the decoupled form, the UAV controller should be modified to compensate for the cable mass as in \cite{Dai2014}.
\subsection{Platform}
The locomotive UAV system can be deployed from ships and marine structures. It can be an independent system, and if designed and equipped to work autonomously, it can link itself to the target floating object using an on-board cable and perform manipulation afterward. Such designs have higher mobility, easier deployability, and independence from potentially-bulky buoys.
\section{Conclusion} \label{sec_conclusion}

The novel problem of a marine locomotive UAV system is defined, in which a quadrotor UAV is tethered to a floating buoy to control its surge velocity. The system dynamics are separately modeled for each subsystem including the water medium, the buoy, and the UAV, then combined via the Euler-Lagrange formulation. 
The attainable setpoints and constraints of the proposed system are defined, then a precision motion control system is designed to manipulate the surge velocity of the buoy within certain limits, which require maintaining the cable in a taut state and keeping the buoy in contact with the water surface. 
A simulation environment is defined, and the proposed SVCS is validated and compared to a nominal Cartesian-based UAV controller, while showing superior tracking performance and disturbance rejection in certain waves, surface currents, and wind conditions.
\par
The proposed system paves the way in front of a wide variety of novel marine applications for multirotor UAVs, where their high speed and maneuverability, as well as their ease of deployment and wide field of vision, give them a superior advantage. It best suits applications that require remote and fast manipulation with minimal water surface disruption. 

Going forward, we aim to extend the problem to the three-dimensional (3D) space, perform energy minimization techniques for the system's path-planning under different wave and current conditions, and study the buoy's stability by introducing additional constraints to the system.
The controller and buoy stability will be further challenged in more complex wave scenarios to uncover the buoy's shape effects on the system performance.
A parallel future path entails building a prototype of the proposed system and performing experimental validation of the proposed controller.
\section*{ACKNOWLEDGMENT}

This work is supported by the University Research Board (URB) at the American University of Beirut (AUB).
\bibliography{main}
\end{document}